\documentclass[12pt]{article}
\usepackage{mathrsfs}
\usepackage{amssymb, amsthm, amsmath}
\usepackage{graphicx}
\usepackage{verbatim}
\usepackage{color}
\usepackage{authblk}

\newcommand{\E}{\mathcal E}
\newcommand{\B}{\mathcal B}
\newcommand{\bE}{\mathbb E}

\newtheorem{thm}{Theorem}

\newtheorem{lem}[thm]{Lemma}

\newtheorem{prop}[thm]{Proposition}

\newtheorem{remark}{Remark}

\newtheorem{assumption}{Assumption}

\title{Capacity dependent analysis for functional online learning algorithms}

\author[1]{Xin Guo}
\author[2]{Zheng-Chu Guo}
\author[3]{Lei Shi}
\affil[1]{School of Mathematics and Physics, The University of
	Queensland, Brisbane, QLD 4072, Australia. xin.guo@uq.edu.au}
\affil[2]{School of Mathematical Sciences, Zhejiang University, Hangzhou
	310027, P.\ R.\ China. guozhengchu@zju.edu.cn}
\affil[3]{School of Mathematical Sciences and Shanghai Key Laboratory for
Contemporary Applied Mathematics, Fudan University, Shanghai 200433,
P.\ R.\ China. leishi@fudan.edu.cn}

\numberwithin{equation}{section}

\begin{document}
\date{}
\maketitle
\begin{abstract}
This article provides convergence analysis of online stochastic gradient
descent algorithms for functional linear models.
Adopting the characterizations of the slope function regularity, the kernel
space capacity, and the capacity of the sampling process covariance operator,
significant improvement on the convergence rates is achieved.
Both prediction problems and estimation problems are studied,
where we show that capacity assumption can alleviate the saturation of the
convergence rate as the regularity of the target function increases.
We show that with properly selected kernel, capacity assumptions
can fully compensate for the regularity assumptions for prediction problems
(but not for estimation problems).
This demonstrates the significant difference between the prediction problems and the
estimation problems in functional data analysis.
\end{abstract}



\section{Introduction}\label{section: introduction}

In this paper, we consider a functional linear model
\begin{gather}
Y = \int_{\mathcal{T}} \beta^*(u)X(u) du + \varepsilon.
\label{model0002LJmHH}
\end{gather}
Here, $\mathcal{T}$ is a compact subset in a Euclidean space $\mathbb{R}^d$, $X$ is a random
function, $\beta^*$ is an unknown slope function, $\varepsilon$ is a centered
random noise with finite variance $\sigma^2=\mathrm{Var}(\varepsilon)<\infty$,
and $Y\in\mathbb{R}$ is the response. We write $(L^2(\mathcal{T}), \left<
\cdot,\cdot \right>_2, \left\| \cdot \right\|_2)$ the space of square
integrable functions on $\mathcal{T}$, and assume $X,\beta^*\in
L^2(\mathcal{T})$. Then, Model (\ref{model0002LJmHH}) can be equivalently
written as
$Y = \left< \beta^*, X \right>_2+\varepsilon$. Without loss of generality, we assume that $\mathcal{T}=[0,1]^d$ throughout the paper.

We study two kinds of learning problems for Model
(\ref{model0002LJmHH}). The \textit{estimation problem} asks one to recover the
unknown slope function $\beta^*$, and the \textit{prediction problem} asks one
to recover the linear functional on $L^2(\mathcal{T})$, denoted by $\varphi^*$, which is given by
\begin{gather}
\varphi^*: f\mapsto\left< \beta^*, f \right>_2=\int_{\mathcal{T}} \beta^*(u)
f(u) du. \label{predictDef000EMnr0}
\end{gather}
Mathematically, $\varphi^*$ is defined with $\beta^*$, which in turn is fully
determined by $\varphi^*$ through the Riesz representation theorem.
Nonetheless, it is well understood that the two learning problems are
different. In particular, the integral in (\ref{predictDef000EMnr0}) brings
a smoothing effect, leading to a weaker regularity requirement for the prediction
problems \cite{CaiHall2006-MR2291496, ChenTangFanGuo2022-MR4388513}.

Write $D = \{ (x_t, y_t) \}_{t=1}^T$ a sample of independent copies of $(X,Y)$
in Model (\ref{model0002LJmHH}). We study both the case of a finite sample
$T<\infty$, and the case $T=\infty$ where $D$ models an ongoing indefinite
sampling process.

Both prediction and estimation problems can be solved by constructing an
estimator $\hat{\beta}$ of the slope function $\beta^*$. In the literature,
many works have been done on functional principal component analysis (FPCA)
\cite{ramsay2005fitting, CaiHall2006-MR2291496, HallHorowitz2007-MR2332269}. FPCA defines
$\hat{\beta}$ with a linear combination of the estimated eigenfunctions of $C$,
which is the covariance function of the random function $X$. Another approach
of constructing $\hat{\beta}$ is the kernel method, which adopts a reproducing
kernel $K$ and represents $\hat{\beta}$ by the linear combination of kernel
functions \cite{YuanCai2010-MR2766857, CaiYuan2012-MR3010906}.

We adopt the kernel method and define $\hat{\beta}$ through stochastic gradient
descent approach in this paper. A reproducing kernel $K$ on $\mathcal{T}$ is
defined as a function $K: \mathcal{T}\times \mathcal{T} \to \mathbb{R}$ that is
symmetric (i.e.\ $K(u,v)=K(v,u)$ for any $u,v\in\mathcal{T}$) and positive
semi-definite (which requires that the Gram matrix $(K(u_i,u_j))_{i,j=1}^n$ is
positive semi-definite for any $n\geq 1$ and any
$u_1,\ldots,u_n\in\mathcal{T}$). We further assume that $K$ is continuous,
exclude the trivial case $K\equiv 0$, and
let $(\mathcal{H}_K, \left< \cdot,\cdot \right>_K, \left\| \cdot \right\|_K)$
denote the reproducing kernel Hilbert space (RKHS) associated with $K$
\cite{CuckerZhou2007-MR2354721, SteinwartChristmann2008-MR2450103}. The
stochastic gradient descent algorithm defines a sequence $\{ \hat{\beta}_t \}$
of estimators, from $\hat{\beta}_1=0$ and then iteratively by
\begin{gather}\label{defIteration000rrCDVx}
\hat{\beta}_{t+1} = \hat{\beta}_t - \eta_t\left( \int_{\mathcal{T}}
\hat{\beta}_t(u) x_t(u) du - y_t \right)
\int_\mathcal{T} K(v, \cdot) x_t(v) dv,\quad \mbox{for }t\geq 1.
\end{gather}
Here $\eta_t>0$ is the step-size. Based on the nature of the sample $D$, we
study two settings of the step-sizes $\{ \eta_t \}$.
\begin{itemize}
\item The \textit{online} setting. We write $|D|=\infty$ and use $D$ to model
the outcome of an ongoing and indefinite sampling process. The estimator
$\hat{\beta}$ is being updated following the sampling process. For example, we
update
the estimator to $\hat{\beta}=\hat{\beta}_{t+1}$ after $t$ steps of iterations
and before the observation $(x_{t+1}, y_{t+1})$ is available. For the online
setting, the step-sizes $\{ \eta_t \}$ are designed to decrease, rendering
Algorithm (\ref{defIteration000rrCDVx}) more and more conservative against the
possible random noise brought by new observations.
\item The \textit{finite-horizon} setting. We assume a finite sample $D$ with
size $|D|=T<\infty$. A constant step-size $\eta_t\equiv \eta = \eta(T)$ is
adopted throughout the iterations (\ref{defIteration000rrCDVx}) with
$t=1,\ldots,T$. The sample $D$ is then exhausted and we use
$\hat{\beta}=\hat{\beta}_{T+1}$ as the derived estimator. The step-size
$\eta(T)$ can be optimized (at least asymptotically) over $T$, but it could be
not trivial later to warm-start the iteration efficiently when new sample
points are available.
\end{itemize}

To measure the estimation performance of $\hat{\beta}$, we use the expected
squared $\mathcal{H}_K$ norm $\mathbb{E}[ \| \hat{\beta} - \beta^* \|^2_K ]$.
Write $\hat{\varphi}: L^2(\mathcal{T})\to L^2(\mathcal{T})$ the estimator of
the functional $\varphi^*$,
\begin{gather*}
\hat{\varphi}: f\mapsto \langle \hat{\beta}, f \rangle_2=\int_\mathcal{T}
\hat{\beta}(u) f(u) du.
\end{gather*}
The prediction performance of $\hat{\varphi}$ is measured by the expected
excess generalization error $\mathbb{E}[\mathcal{E}(\hat{\varphi})]$.
Here for any linear functional $\varphi$ on
$L^2(\mathcal{T})$,
\begin{gather*}
\mathcal{E}(\varphi) = \mathbb{E}[(Y - \varphi(X))^2 - (Y - \varphi^*(X))^2],
\end{gather*}
where the expectation is taken with respect to the distribution of $(X,Y)$ in
Model (\ref{model0002LJmHH}).

As a technical instrument, the integral operator $L_K: L^2(\mathcal{T})\to
L^2(\mathcal{T})$ is defined with the reproducing kernel $K$, by
\begin{gather}\label{defLK000t2R8Y}
L_K(f) = \int_\mathcal{T} K(\cdot, u) f(u) du.
\end{gather}
It is well understood in the literature that $L_K$ is positive semi-definite
(thus self-adjoint), and of trace class (so, compact). See, e.g.,
\cite[Theorem 4.27]{SteinwartChristmann2008-MR2450103}.
The power $L_K^r$ with $r\in(0,\infty)$ is well defined by the
spectral theorem. In terms of $L_K$, the iteration
(\ref{defIteration000rrCDVx}) is equivalently written as
\begin{gather}\label{eqvDefIteration0006PgUy}
\hat{\beta}_{t+1} = \hat{\beta}_t - \eta_t \left( \langle \hat{\beta}_t,
x_t\rangle_2 - y_t \right) L_K x_t, \quad \mbox{for }t\geq 1.
\end{gather}

For the sake of simplicity we assume $\mathbb{E}[X] = 0$ and $\|X\|_2=1$ a.s.
Consequently, $\mathbb{E}[Y] = 0$. The covariance function $C$ has the
form
\begin{gather*}
C(u,v) = \mathbb{E}[X(u)X(v)],\quad \mbox{for }u,v\in\mathcal{T}.
\end{gather*}
Obviously $C$ is also a reproducing kernel. We further assume that $C$ is
continuous, exclude the trivial case $C\equiv 0$, and define the operator
$L_C$ on $L^2(\mathcal{T})$ in the same way
as (\ref{defLK000t2R8Y}) by substituting $K$ with $C$. So, $L_C$ is
self-adjoint, positive semi-definite, of trace class, and thus compact. The
power $L_C^r$ with $r>0$ is well defined. For any $f,g,h\in L^2(\mathcal{T})$,
we define $f\otimes g$ as a rank-one operator on $L^2(\mathcal{T})$ defined by
$(f\otimes g)h=\left< g,h \right>_2f$. For any
linear functional $\varphi(\cdot) = \left< \beta,\cdot \right>_2$ on
$L^2(\mathcal{T})$, the excess generalization error can be written in terms of
the norm of $L^2(\mathcal{T})$,
\begin{align}
\mathcal{E}(\varphi)&= \mathbb{E}\left[ (Y - \langle
\beta,X\rangle_2)^2 - (Y - \langle \beta^*,X\rangle_2)^2 \right]\nonumber\\
&=\mathbb{E}\left[ \langle \beta-\beta^*,X\rangle_2^2 \right]=
\mathbb{E}[\langle \beta-\beta^*, X\otimes X(\beta-\beta^*) \rangle_2]
\nonumber\\
&=\|L_C^{1/2}
(\beta-\beta^*)\|^2_2. \label{excessGEbyNorm000itb0ac}
\end{align}

Since $\mathcal{T}$ is compact, we write
\begin{gather*}
\kappa = \max_{u\in\mathcal{T}}\sqrt{K(u,u)}\in(0, \infty).
\end{gather*}
Recall that by the positive semi-definiteness, $|K(u,v)|\leq
\sqrt{K(u,u)K(v,v)}\leq \kappa^2$ for any $u,v\in\mathcal{T}$. The spectral
norm of $L_K$ is bounded by $\| L_K \|_{\mathsf{op}(L^2)}\leq \kappa^2$.

Modern scalable computing and stochastic optimization techniques make stochastic gradient descent a popular approach across various applications. Theoretical analysis of its convergence is also the subject of an intense recent study. The present work aims to establish a novel capacity-dependent convergence analysis for stochastic gradient descent \eqref{defIteration000rrCDVx} which is applied to solve the linear functional model \eqref{model0002LJmHH} in an RKHS. We study prediction problem through the convergence of excess generalization error \eqref{excessGEbyNorm000itb0ac} and estimation problem through the strong convergence in an RKHS. Our analysis developed in this paper leads to fast rates for both types of convergence.
State-of-the-art convergence rates in RKHS metric are obtained.
From the viewpoint of approximation, this kind of convergence is much stronger, which ensures that the estimators can approximate the underlying target itself and its derivatives as well \cite{Zhou2008-MR2444183}. Our error estimates fully exploit the spectral structure of the operators and the capacity condition encoding the smoothness of kernels and covariance function.
Our work provides insights for the applications of kernel
methods to functional data analysis, and better understanding of the difference
between the estimation problems and the prediction problems in functional
linear models.

The rest of this paper will be organized as follows. We present the main results in Section \ref{section: main results}. Discussions and comparisons with related works are given in Section \ref{secDiscussionOfAssumptions000aYqK3}.  Section \ref{section: error decomposition} introduces two novel error decomposition formulas of the algorithm (\ref{defIteration000rrCDVx}). The proofs of main results are postponed to Section \ref{section: bounding the excess generalization error} and Section \ref{section: bounding the estimation error} after some preliminary estimates established for the convergence analysis.

\section{Main Results}\label{section: main results}

In this section we list some main assumptions and present the convergence rates
of the stochastic gradient descent algorithm (\ref{defIteration000rrCDVx}),
in the finite-horizon and online settings, respectively. We provide
discussions of the assumptions in Section
\ref{secDiscussionOfAssumptions000aYqK3}.

Denote $\mathscr{L}_K=L_C^{1/2}L_KL_C^{1/2}$ and $\mathscr{L}_C
=L_K^{1/2}L_CL_K^{1/2}$. It is easy to verify that both of the operators
$\mathscr{L}_K$ and $\mathscr{L}_C$ are self-adjoint, positive semi-definite,
of trace class, and compact.

\begin{assumption}[Regularity Condition of the slope $\beta^*$]
\label{assumption1}
There exists some $g^*$ in $L^2(\mathcal{T})$ and $r\in(0, \infty)$ such that
\begin{gather*}
L_C^{1/2}\beta^* = \mathscr{L}^r_K g^*.
\end{gather*}
\end{assumption}

For any positive semi-definite compact operator $L$, let $\mathrm{Tr}(L)$
denote the trace of $L$, i.e., the sum of all the positive eigenvalues
(counting multiplicity) of $L$. In particular, $\mathrm{Tr}(L)<\infty$ if and
only if $L$ is of trace class.

\begin{assumption}[Capacity Condition]\label{assumption2}
\begin{gather*}
\mathrm{Tr}(\mathscr{L}^s_K)<\infty,\quad \mbox{for some }0<s\leq 1.
\end{gather*}
\end{assumption}
Note that since $\mathscr{L}_K$ is a trace-class operator, Assumption 2 with
$s=1$ holds true automatically.

\begin{assumption}[Moment Condition]\label{assumption3}
For Model (\ref{model0002LJmHH}), there exist constant $c_{\mathsf{M}}>0$
such that for any $f$ in $L^2(\mathcal{T})$,
\begin{align}\label{eqAssumption3}
\mathbb{E}[\left< X,f \right>_2^4] \leq c_{\mathsf{M}}
\left( \mathbb{E}[\langle X,f \rangle_2^2] \right)^2.
\end{align}
\end{assumption}

\subsection{Analysis of the Prediction Error}\label{subsection: analysis of the prediction error}

In this subsection, we study the estimator $\hat{\varphi} = \langle
\hat{\beta}, \cdot \rangle_2$ for the prediction problem and bound the expected
excess generalization error.

\begin{thm}\label{thm: decreasing step size capacity dependent L2}
In the online setting, define $\{ \hat{\varphi}_t = \langle \hat{\beta}_t,
\cdot \rangle_2 \}$ through (\ref{defIteration000rrCDVx}). Under
Assumptions \ref{assumption1} (with $r>0$), \ref{assumption2} (with $0<s\leq 1$),
and \ref{assumption3}, set $\eta_t = \eta_0 t^{-\theta}$ with
\begin{gather}\label{temp000sEyvM}
\theta=\frac{\min\{2r, 2-s\}}{1 + \min\{2r, 2-s\}}=
\left\{
\begin{array}{ll}
\frac{2r}{2r+1},&\mbox{when }2r\leq 2-s,\\
\frac{2-s}{3-s},&\mbox{when }2r\geq 2-s.
\end{array}
\right.
\end{gather}
If $0<\eta_0\leq \min\{ 1, \kappa^{-2}, C^{\mathsf{S}}_1 \}$ (where
$C^\mathsf{S}_1$ is a constant, and it will be specified by
(\ref{specifyC1S000hVp9yJ}) in the proof), then
\begin{gather}
\mathbb{E}[\mathcal{E}(\hat{\varphi}_{t+1})] \leq C_1 \left\{\begin{array}{ll}
(t+1)^{-\theta},&0<s<1,\\
(t+1)^{-\theta}\log(t+1),& s=1,
\end{array}\right.
\quad \mbox{for any } t\geq 1,\label{boundOL0002YN04Pgm}
\end{gather}
where $C_1$ is a constant independent of $t$, and it will be specified
by (\ref{C1Def0001LaZ}) in the proof.
\end{thm}

For the piecewise definition (\ref{temp000sEyvM}), we let the domains overlap
on purpose to highlight the continuity of $\theta$ on the
whole domain $r>0$ and $0<s<1$.
The index $\theta$ as a function of $r$ and $s$ is also
visualized in Figure \ref{figVisualizeTheta000bKc4nH}.
Without Assumption \ref{assumption2} (i.e., case $s=1$ in
(\ref{boundOL0002YN04Pgm})), the convergence rate $O((t+1)^{-2r/(2r+1)}
\log(t+1))$, saturated as $O((t+1)^{-1/2}\log(t+1))$ for $r\geq 1/2$, is also
obtained in \cite{ChenTangFanGuo2022-MR4388513}. Here, $r$ indicates the
regularity of the target function $\beta^*$ as described in Assumption
\ref{assumption1}. The saturation means beyond $r\in(0,1/2]$, further
improvement of such regularity (i.e., increasing of $r$)
does not help to improve the rate $\mathbb{E}
[\mathcal{E}(\hat{\varphi}_{t+1})]$ converges to zero.
In this paper, Theorem \ref{thm: decreasing step size capacity dependent L2}
suggests that Assumption \ref{assumption2} on capacity with $s <1$,
not only removes the logarithmic factor in the convergence rates, but also
uplifts the saturating boundary from $1/2$ to $1/2+(1-s)/2$.


\begin{figure}
\begin{center}
\includegraphics[width=7cm]{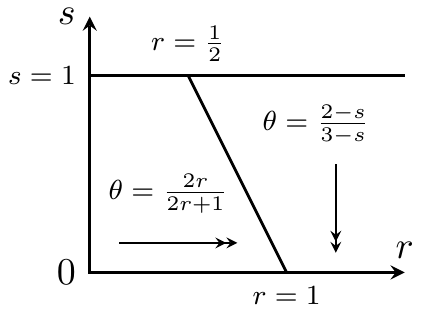}
\caption{The index $\theta$ in Theorem \ref{thm: decreasing step size capacity
dependent L2} as a function of $(r,s)$. Here the double-headed arrows show the
gradient directions.}
\label{figVisualizeTheta000bKc4nH}
\end{center}
\end{figure}

In the following, Theorem \ref{thm: convergence rate in L2 constant step size and capacity independent}
shows that in the finite-horizon setting, Algorithm (\ref{defIteration000rrCDVx})
does not suffer from the above discussed saturation, and the expected prediction
error converges to zero in a rate arbitrarily close to $O(T^{-1})$, for
sufficiently large $r$.

\begin{thm}\label{thm: convergence rate in L2 constant step size and capacity
independent}
In the finite-horizon setting with $1\leq T=|D|<\infty$, define
$\hat{\varphi}_{T+1} = \langle \hat{\beta}_{T+1}, \cdot \rangle_2$ through
(\ref{defIteration000rrCDVx}). Under Assumptions \ref{assumption1} (with $r>0$),
\ref{assumption2} (with $0<s\leq 1$), and \ref{assumption3}, set the constant
step-size $\eta_t = \eta_0T^{-2r/(2r+1)}$, with $0<\eta_0\leq \min\{ 1,
\kappa^{-2}, C_2^{\mathsf{S}} \}$ (where $C_2^{\mathsf{S}}$ is a constant
independent of $T$, and it will be specified by (\ref{specifyC2S000zO2SiQ})
in the proof). Then,
\begin{gather}
\label{constSSErrBound000nCCioi5s}
\mathbb{E}[\mathcal{E}(\hat{\varphi}_{T + 1})] \leq C_2\left\{\begin{array}{ll}
T^{-2r/(2r+1)},&\mbox{when }0<s<1,\\
T^{-2r/(2r+1)}\log(T+1),&\mbox{when }s=1.
\end{array}\right.
\end{gather}
where the constant $C_2$ is independent of $T$, and it will be specified
by (\ref{specifyC2000lBOk5F}) in the proof.
\end{thm}
The capacity independent convergence rate $O(T^{-2r/(2r+1)}\log(T+1))$
for $s=1$ in (\ref{constSSErrBound000nCCioi5s}) is first derived in
\cite{ChenTangFanGuo2022-MR4388513}. In the finite-horizon setting,
the capacity assumption $0<s<1$ helps remove the logarithmic factor.

\subsection{Analysis of the Estimation Error}\label{subsection: analysis of the estimation error}

In this subsection, we study the estimator $\hat{\beta}$ for the estimation
problem. The analysis employs the following Assumption \ref{assumption4} to
replace Assumption \ref{assumption1}.
\begin{assumption}[Regularity Condition of the slope $\beta^*$]
\label{assumption4}
There exists some $g^\dag$ in $L^2(\mathcal{T})$ and $r>0$, such
that
\begin{gather*}
\beta^* = L_K^{1/2} \mathscr{L}_C^r g^\dag.
\end{gather*}
\end{assumption}

This assumption implies that the slope $\beta^*$ lies in the range of $L^{1/2}_K$, i.e., $\beta^* \in \mathcal{H}_K$.

\begin{thm} \label{thm3HKOnline000EGE8Dv}
In the online setting, define $\{ \hat{\beta}_t \}_{t\geq 1}$ through
(\ref{defIteration000rrCDVx}). Under Assumptions 2 (with $0<s<1$), 3,
and 4 (with $r>0$), set step-sizes $\eta_t=\eta_0 t^{-\theta}$ with
$0<\eta_0\leq \min\{ 1, \kappa^{-2}, C_3^\mathsf{S} \}$ (where
$C_3^{\mathsf{S}}$ is a constant independent of $t$, and it will be specified
by (\ref{defC3S000vypUI}) in the proof), and
\begin{gather}\label{temp000Eve1p}
\theta=\left\{
\begin{array}{ll}
\frac{2r+s}{2r+s+1},&\mbox{when }2r\leq 1-s,\\
1/2,&\mbox{when }2r\geq 1-s.
\end{array}
\right.
\end{gather}
Then,
\begin{gather}\label{estimationonline}
\mathbb{E}[\|\hat{\beta}_{t+1} - \beta^*\|_K^2]\leq C_3 \left\{
\begin{array}{ll}
	(t+1)^{-2r/(1+s+2r)},&2r< 1-s,\\
	(t+1)^{-(1-s)/2}\log(t+1),&2r\geq 1-s,
\end{array}
\right. \quad \mbox{for any } t\geq 1,
\end{gather}
where $C_3$ is a constant independent of $t$, and it will be specified
by (\ref{C3Specify000JTif}) in the proof.
\end{thm}
In the definition (\ref{temp000Eve1p}), $\theta$ is a continuous function of $r>0$ and
$0<s<1$. So we purposely use two overlapping domains. The power index of the rates
in (\ref{estimationonline}) will be elucidated in Figure \ref{figThm3000rskhh}.
\begin{figure}[h]
\centering
\includegraphics[width=0.4\textwidth]{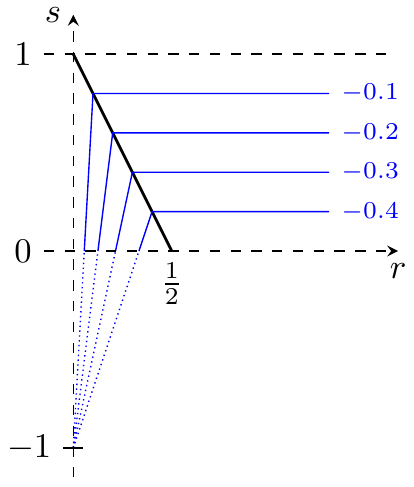}
\caption{The power index $\tilde{\omega}(r,s)$ of the convergence estimate
(\ref{estimationonline}) in Theorem \ref{thm3HKOnline000EGE8Dv}. The thick
black line marks the boundary $2r=1-s$ of the two regimes.
$\tilde{\omega}=-2r/(1+s+2r)$ when $2r\leq 1-s$, and $\tilde{\omega}=-(1-s)/2$
when $2r\geq 1-s$.
Contours of
$\tilde{\omega}$ are plotted in solid blue lines, and further extended by
dotted lines.}
\label{figThm3000rskhh}
\end{figure}

\begin{remark}
Theorem \ref{thm3HKOnline000EGE8Dv} does not work in the capacity independent
setting $s=1$, where the convergence analysis remains an open problem.
\end{remark}


Next we establish unsaturated convergence rates of estimation error for the finite-horizon setting.

\begin{thm}\label{thm4HKFH0000KvQ6A}
In the finite-horizon setting with $1\leq T=|D|<\infty$, define
$\{\hat{\beta}_t: 1\leq t\leq T+1\}$ through (\ref{defIteration000rrCDVx}).
Under Assumptions \ref{assumption2} (with $0<s\leq 1$), 3, and 4 (with $r>0$),
set the constant step-size $\eta_t= \eta_0 T^{-(s+2r) / (1+s+2r)}$, with
$0<\eta_0\leq \min\{ 1, \kappa^{-2}, C_4^{\mathsf{S}} \}$ (where
$C_4^{\mathsf{S}}$ is a constant independent of $T$, and it will be specified
by (\ref{specifyC4S000KtAVF}) in the proof). Then,
\begin{gather}\label{estimationfinitehorizon}
\mathbb{E}[\|\hat{\beta}_{T+1} - \beta^*\|_K^2] \leq C_4 T^{-2r/(1 + s + 2r)},
\end{gather}
where the constant $C_4$ is independent of $T$, and it will be specified
by (\ref{C4Specify000FQapF}) in the proof.
\end{thm}

%

The rates (\ref{estimationfinitehorizon}) in Theorem \ref{thm4HKFH0000KvQ6A} does not saturate for $r>0$, and are arbitrarily close to $O(T^{-1})$ for sufficiently large $r$. For a fixed $r>0$, with a smaller $s$ one has a faster rate in (\ref{estimationfinitehorizon}). Here, a smaller $s$ indicates a stronger capacity assumption (Assumption \ref{assumption2}). As we shall see in Theorem \ref{thmExpandTrace000VtDr04} below, a smaller $s$ corresponds to faster eigenvalue decay for $\mathscr{L}_K$ (equivalently, faster eigenvalue decay for $\mathscr{L}_C$), and a smaller hypothesis space $L_K^{1/2}\mathscr{L}_C^r(L^2(\mathcal{T}))$ in Assumption \ref{assumption4}.

\section{Comparisons and Discussions}\label{secDiscussionOfAssumptions000aYqK3}

There has been rapidly growing literature focusing on stochastic gradient descent and its variants in an RKHS or general Hilbert spaces \cite{YingPontil2008-MR2443089,dieuleveut2016nonparametric,pillaud2018statistical,guo2019fast,berthier2020tight,guo2022rates}. We refer the readers
to these papers and references therein. Our paper contributes to the theoretical analysis of functional linear regression in an RKHS that stems from the
the work of Yuan and Cai which establishes capacity dependent analysis for batch learning \cite{YuanCai2010-MR2766857, CaiYuan2012-MR3010906}. As far as we know, the convergence of stochastic gradient descent  has not been investigated in the context of functional linear regression in an RKHS till the very recent paper \cite{ChenTangFanGuo2022-MR4388513} in which the authors conduct capacity independent analysis of the prediction error.

Under the batch leaning setting, Yuan and Cai \cite{YuanCai2010-MR2766857} derive the minimax optimal
convergence rate $T^{-2s_*/(2s_*+1)}$ of the excess generalization error
$\mathcal{E}(\hat{\varphi}_{T+1})$ for prediction, with the regularity assumption
$\beta^*\in\mathcal{H}_K$ and capacity assumption on the rates of eigenvalue
decay, $\lambda_i(L_K)\sim i^{-2s_1}$ (here $a_i\sim b_i$ means $a_i/b_i$ being
uniformly bounded away from zero and infinity as $i\to\infty$)
and $\lambda_i(L_C)\sim i^{-2s_2}$, where
$s_1,s_2>1/2$ and $s_*=s_1+s_2$. Later, Cai and Yuan
\cite{CaiYuan2012-MR3010906} derive the same rate with a different capacity
assumption $\lambda_i(\mathscr{L}_C)\sim i^{-2s_*}$.

Compared with these works,
the strength of our analysis includes that first, our Assumption \ref{assumption2} on capacity,
$\mathrm{Tr}(\mathscr{L}_K^s)<\infty$, is way more general. We shall
see in Theorem \ref{thmExpandTrace000VtDr04} that this is roughly equivalent to the
assumption $\lambda_i(\mathscr{L}_K)=O(i^{-1/s})$. We shall see in Remark
\ref{decreasingCounterExamp000pSLt} that although the eigenvalues $\{
\lambda_i(\mathscr{L}_K) \}_{i=1}^{\infty}$ are arranged non-increasingly,
in general there is no exact index $s_*$ such that $\lambda_i(\mathscr{L}_K)
\sim i^{-2s_*}$ (same for other compact operators including $L_K$, $L_C$,
and $\mathscr{L}_C$). Second, our analysis supports finer characterizations
$L_C^{1/2}\beta^*=
\mathscr{L}_K^r(g^*)$ (Assumption \ref{assumption1}) and $\beta^*=L_K^{1/2}
\mathscr{L}_C^r (g^\dag)$ (Assumption \ref{assumption4}) of slope function
regularity. This leads to a better convergence rate $O(T^{-2r/(2r+1)})$ in
Theorem \ref{thm: convergence rate in L2 constant step size and capacity
independent}, than $O(T^{-2s_*/(2s_*+1)})$ when $r>s_*$. Third, we proved
the non-trivial convergence rates for the estimation error $\|\hat{\beta}-
\beta^*\|_K^2$ in $\mathcal{H}_K$ metric, $O(T^{-2r/(2r+1)})$ (saturated
at $r=(1-s)/2$) in the online setting in Theorem \ref{thm3HKOnline000EGE8Dv},
and $O(T^{-2r/(1+s+ 2r)})$ in the finite-horizon setting in Theorem
\ref{thm4HKFH0000KvQ6A}. Note that the analysis in
\cite{YuanCai2010-MR2766857} just provides a constant rate $O(1)$ for
$\|\hat{\beta}-\beta^*\|_K^2$.

Next we provide some comments on the main assumptions in Section
\ref{section: main results}. For any bounded self-adjoint operators $A$ and
$B$ on $L^2(\mathcal{T})$, we write $A\succeq B$ (or $B \preceq A$)
if $A - B$ is positive semi-definite.

\begin{remark}
It is well understood \cite[Remark 1]{ChenTangFanGuo2022-MR4388513} that when
$0<r<1/2$, if $L_K^{\tau}\succeq \delta L_C^\nu$ for some $\tau,\delta,
\nu>0$ with $\tau+\nu\geq 1$ and $r = \tau/(2\tau+2\nu)$, then Assumption
\ref{assumption1} is guaranteed by any $\beta^*\in L^2 (\mathcal{T})$.
That is, with a carefully selected reproducing kernel $K$,
for the prediction error to converge,
the capacity assumption (Assumption \ref{assumption2}) can fully compensate
for the regularity assumption (Assumption \ref{assumption1}).
Note that the above condition $L_K^{\tau}\succeq \delta L_C^\nu$ puts some
requirement on the selection of the reproducing kernel $K$, but it does not
require the one-to-one matching between the eigenfunctions of $L_K$ and $L_C$,
respectively.

Similarly, if $L_C\succeq \delta L_K^{\nu}$ for some $\delta,
\nu>0$, then Assumption \ref{assumption4} with $0<r\leq 1/2$ is guaranteed when
$\beta^*\in L_K^{r(1+\nu)+\frac12} (L^2(\mathcal{T}))$. However, Assumption
\ref{assumption4} implies $\beta^*\in L_K^{1/2}(L^2(\mathcal{T}))$. Therefore,
the regularity assumption for the estimation error to converge,
can not be fully compensated for by the capacity assumption.
This demonstrates a significant difference between the prediction problems, and
the estimation problems in functional data analysis.
\end{remark}

In the literature of kernel-based learning algorithms
\cite{caponnetto2007optimal,bauer2007regularization, blanchard2010optimal,
lin2017distributed,GuoShi2019-MR3994990, GuoLinZhou2017-MR3672238,
WangShengMFC2022-OfficialSite,HeSunMFC2022-OfficialSite},
the capacity of the hypothesis space $\mathcal{H}_K$
is usually measured by covering numbers, or 
the effective dimension $\mathcal{N}_{L_K}(\lambda) =
\mathrm{Tr}((L_K+\lambda I)^{-1}L_K)$, where $I$ denotes the identity operator.
A typical capacity assumption takes the form $\mathcal{N}_{L_K}(\lambda) =
O(\lambda^{-s})$ (as $\lambda \downarrow 0$) for some $0<s<1$, and is well
understood. The following theorem shows that roughly speaking, Assumption
\ref{assumption2} with $0<s<1$ is comparable to the assumption
$\mathcal{N}_{\mathscr{L}_K}(\lambda) = O(\lambda^{-s})$ as $\lambda\downarrow
0$. The conclusion is well understood \cite{guo2019fast,guo2022rates}, but the proof through \eqref{expandTrace000GSDbo}, is to our best knowledge not available elsewhere.

\begin{thm}\label{thmExpandTrace000VtDr04}
Let $L$ be a positive semi-definite operator of trace class with infinite
positive eigenvalues $\{ \lambda_i=\lambda_i(L) \}_{i=1}^\infty$ arranged in
non-increasing order. Let $0<s<1$. We have
\begin{gather}\label{expandTrace000GSDbo}
\mathrm{Tr}(L^s) = \frac{\sin(\pi s)}{\pi} \int^\infty_0 \lambda^{s-1}
\mathcal{N}_L(\lambda)d \lambda.
\end{gather}
Consequently,
\begin{itemize}
\item[(a).] If $\mathrm{Tr}(L^s)<\infty$, then $\mathcal{N}_L(\lambda)
=O(\lambda^{-s})$ as $\lambda\downarrow 0$;
\item[(b).] If $\mathcal{N}_L(\lambda)=O(\lambda^{-s})$ as $\lambda\downarrow
0$, then $\mathrm{Tr}(L^{s+\epsilon})<\infty$ for any $\epsilon>0$;
\item[(c).] Moreover, for any fixed $0<s<1$, $\mathcal{N}_L(\lambda)=
O(\lambda^{-s})$ as $\lambda\downarrow 0$ if and only if $\lambda_i =
O(i^{-1/s})$ as $i\to\infty$.
\end{itemize}
\end{thm}

The relations listed in Theorem \ref{thmExpandTrace000VtDr04} are summarized in
Figure \ref{FigRelations000cSiSsXH}.

\begin{figure}
\begin{center}
\includegraphics[width=7cm]{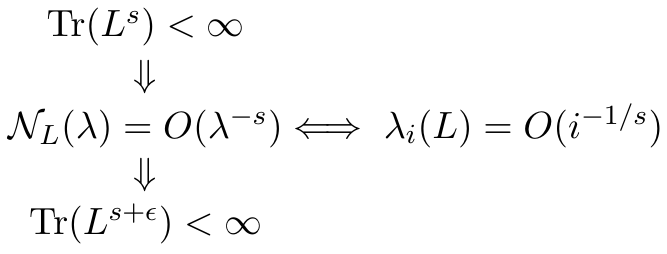}
\caption{Summary of the relations listed in Theorem
\ref{thmExpandTrace000VtDr04}.}
\label{FigRelations000cSiSsXH}
\end{center}
\end{figure}

The case $L$ has only finite positive eigenvalue is trivial, where
$\mathcal{N}_L(\lambda) = O(1)$ as $\lambda\downarrow 0$ and
$\mathrm{Tr}(L^s)<0$ for any $s>0$.

Note that on the one hand, when $L^s$ does not belong to the trace class
(equivalently, $\mathrm{Tr}(L^s)=\infty$), Equation (\ref{expandTrace000GSDbo})
implies that the integral on its right-hand side diverges to infinity. On the
other hand, when this integral diverges to infinity, $\mathrm{Tr}(L^s)=\infty$.

The bound $\mathcal{N}_L(\lambda) = O(\lambda^{-s})$ as $\lambda\downarrow 0$
does not guarantee $\mathrm{Tr}(L^s)<\infty$. For example,
$\lambda_i=i^{-1/s}$ implies
$\mathrm{Tr}(L^s)=\infty$, yet we still have
\begin{gather*}
\mathcal{N}_L(\lambda) = \sum^\infty_{i = 1}\frac{1}{1+\lambda i^{1/s}}\leq
\int^\infty_0 \frac{du}{1+\lambda u^{1/s}} = \lambda^{-s}\int^\infty_{0}
\frac{du}{1 + u^{1/s}} = O(\lambda^{-s}),\quad \mbox{as }\lambda\downarrow 0.
\end{gather*}

\begin{remark}\label{decreasingCounterExamp000pSLt}
Note that in general, for a non-increasing sequence $\{ a_i \}_{k=1}^{\infty}\subset(0,\infty)$,
Theorem \ref{thmExpandTrace000VtDr04} does not suggest the existence of some
$\gamma>0$ such that $a_k\sim k^{-\gamma}$. It is easy to construct a
non-increasing sequence that stays between $k^{-\gamma_1}$ and $k^{-\gamma_2}$ for
any $\gamma_1>\gamma_2>0$. To this end, we define $\{ b_k \}$ as $b_1=2$ and
$b_{k+1} = b_k^{\gamma_1/\gamma_2}$ for $k\geq 1$.
We define a function $f$ on $[2,\infty)$, piece-wisely by
$f(x)=b_k^{-\gamma_1}$ for $b_k\leq x<b_{k+1}$. Writing $a_k=f(k+1)$ to give
\begin{gather*}
\limsup_{k\to\infty} \frac{a_k}{k^{-\gamma_2}}=\liminf_{k\to\infty}
\frac{a_k}{k^{-\gamma_1}} = 1.
\end{gather*}
\end{remark}

\proof[Proof of Theorem \ref{thmExpandTrace000VtDr04}]
Write $B(u,v)$ the Euler beta function for $u,v>0$. Recall that for any $a>0$,
\begin{gather*}
\frac{\pi}{\sin(\pi s)} = B(s,1-s) = \int^\infty_0\frac{\xi^{s-1}}{1+\xi}d\xi
\overset{\xi=\lambda/a}{=\!=\!=\!=} a^{-s}\int^\infty_0
\frac{a\lambda^{s-1}}{a+\lambda}d\lambda.
\end{gather*}
So,
\begin{gather}\label{eqExpandAS0009gJKV}
a^s = \frac{\sin(\pi s)}{\pi}\int^\infty_0 \lambda^{s-1}\frac{a}{a+\lambda}
d\lambda.
\end{gather}
In (\ref{eqExpandAS0009gJKV}) substitute $a$ with all the positive eigenvalues
of $L$ respectively, and take the sum to obtain (\ref{expandTrace000GSDbo}).
Since $L$ is of trace class, $\mathcal{N}_L(\lambda)$ is well defined for each
$\lambda>0$. Obviously $\lambda^{s-1}$ and $\mathcal{N}_L(\lambda)$ are
non-increasing. So when $\mathrm{Tr}(L^s)<\infty$,
\begin{gather*}
\lambda^s \mathcal{N}_L(\lambda)=\lambda^{s-1}\mathcal{N}_L(\lambda)
\int^\lambda_0 d\xi< \int^\infty_0 \xi^{s-1}\mathcal{N}_L(\xi) d\xi =
\frac{\pi \mathrm{Tr}(L^s)}{\sin(\pi s)} < \infty,
\end{gather*}
which verifies (a). Now assume $\mathcal{N}_L(\lambda) = O(\lambda^{-s})$ as
$\lambda\downarrow 0$. Then there are two constants $0<\delta,C_1<\infty$ such
that $0\leq \mathcal{N}_L(\lambda)\leq C_1\lambda^{-s}$ for any $0<\lambda\leq
\delta$. So,
\begin{gather*}
\int^\delta_0\lambda^{s+\epsilon-1}\mathcal{N}_L(\lambda) d\lambda\leq
C_1\int^\delta_0 \lambda^{\epsilon - 1}d\lambda<\infty.
\end{gather*}
Since $L$ is in the trace class, when $s+\epsilon\geq 1$, (b) is trivial. Now
we assume $s+\epsilon<1$. Note that $\mathcal{N}_L(\lambda)\leq
\mathrm{Tr}(L)/\lambda$. So
\begin{gather*}
\int^\infty_{\delta} \lambda^{s+\epsilon-1}\mathcal{N}_L(\lambda) d\lambda
\leq \mathrm{Tr}(L) \int^\infty_{\delta} \lambda^{s+\epsilon - 2}
d\lambda<\infty.
\end{gather*}
The claim (b) is verified by combining the above two bounds together.

Now we verify item (c). When $\lambda_i=O(i^{-1/s})$, there is
some constant $C_2>0$ such that $\lambda_i\leq C_2 i^{-1/s}$ for all $i\geq 1$.
Since $u/(u+\lambda)$ is an increasing function of $u$,
\begin{gather*}
\mathcal{N}_L(\lambda)\leq \sum^\infty_{i=1}
\frac{C_2i^{-1/s}}{C_2i^{-1/s} + \lambda} = \sum^\infty_{i = 1}
\frac{1}{1+\lambda i^{1/s}/C_2}\\
\leq \int^\infty_0 \frac{du}{1 + \lambda u^{1/s}/C_2} =
\left( \frac{\lambda}{C_2} \right)^{-s}\int^\infty_0
\frac{du}{1+u^{1/s}} = O(\lambda^{-s}).
\end{gather*}
This verifies the ``if'' part. For the ``only-if'' part,
when $\mathcal{N}_L(\lambda)=O(\lambda^{-s})$,
it is easy to see that there is some $C_3>0$ such that $\mathcal{N}_L(\lambda)
\leq C_3\lambda^{-s}$ for every $0<\lambda\leq s\lambda_1/(1-s)$. Since for a
fixed $\lambda$, $\{ \lambda_i/(\lambda_i+\lambda) \}^\infty_{i=1}$ is a
non-increasing sequence, $i\lambda_i/(\lambda_i+\lambda)\leq
C_3 \lambda^{-s}$ for each $\lambda\in(0, s\lambda_1/(1-s)]$ and $i\geq 1$.
Therefore,
\begin{gather}
\label{infAchieved000AaFKGl}
i\lambda_i\leq \inf_{\lambda\in (0, s\lambda_1/(1-s)]} C_3\lambda^{-s}
(\lambda_i+\lambda) = C_3 \lambda_i^{1-s}s^s(1-s)^{1-s},
\end{gather}
where the infimum is achieved at $\lambda=s\lambda_i/(1-s) \in (0,
s\lambda_1/(1-s)]$. From (\ref{infAchieved000AaFKGl}) one obtains
$\lambda_i=O(i^{-1/s})$ as $i\to\infty$, and completes the proof.
\qed

Assumption \ref{assumption3} is quite often adopted in the literature of
functional linear regression. For example, if $X$ is a Gaussian process, then
(\ref{eqAssumption3}) is satisfied with $c_{\mathsf{M}}=3$. See
\cite{YuanCai2010-MR2766857, CaiYuan2012-MR3010906, FanLvShi2019-JOURNALHOMEPAGE}.

\section{Error Decomposition}\label{section: error decomposition}

Our analysis starts with error decomposition. By (\ref{eqvDefIteration0006PgUy}), i.e., the equivalent expression of algorithm
(\ref{defIteration000rrCDVx}), for any $t\geq 1$,
\begin{align}
\hat{\beta}_{t+1}-\beta^* &=\hat{\beta}_t-\beta^*-\eta_t(\langle
\hat{\beta}_t,x_t\rangle_2-y_t)L_K x_t\nonumber\\
&=(I-\eta_t L_KL_C)(\hat{\beta}_t-\beta^*)+\mathcal{B}_t,\label{eqn: difference}
\end{align}
where $\mathcal{B}_t=\eta_t(y_t-\langle\hat{\beta}_t,x_t\rangle_2)
L_Kx_t+\eta_t L_K L_C(\hat{\beta}_t-\beta^*)$, of which the second term is the
conditional mean of the first term,
\begin{align}
\mathbb{E}_{z_t}\left[ (y_t - \langle \hat{\beta}_t, x_t \rangle_2)
L_Kx_t \right]
&= \mathbb{E}_{x_t}\left[ \langle \beta^*-\hat{\beta}_t, x_t \rangle_2 L_K x_t
\right] \nonumber\\
&= L_KL_C(\beta^* - \hat{\beta}_t). \label{B_meanZero000uQjHu}
\end{align}
Where $z_t=(x_t,y_t)$, and the expectations $\mathbb{E}_{z_t}$ and
$\mathbb{E}_{x_t}$ are taken with respect to the (conditional) distributions of
$z_t = (x_t,y_t)$ and $x_t$, respectively. Equation (\ref{B_meanZero000uQjHu})
shows that $\mathcal{B}_t$ is mean-zero, $\bE_{z_t}[\B_t]=0$.
Then applying induction to (\ref{eqn: difference}) implies that for any $t\geq
1$,
\begin{align}\label{eqn: error decomposition}
\hat{\beta}_{t+1}-\beta^*=- \left[ \prod_{k=1}^t(I-\eta_k L_KL_C) \right]
\beta^*+\sum_{k=1}^t
\left[ \prod_{j=k+1}^t(I-\eta_jL_K L_C) \right] \mathcal{B}_k,
\end{align}
where and in the following, the product of an empty set of operators is defined
as the identity operator, $\prod_{j=t+1}^t(I-\eta_jL_K L_C)=I$. Recall that
$\mathscr{L}_K = L_C^{1/2}L_KL_C^{1/2}$.

\begin{prop}\label{prop: error decomposition in L2 decreasing step size}
Define $\{ \hat{\varphi}_t=\langle \hat{\beta}_t,\cdot \rangle_2: t\geq 1 \}$
through (\ref{defIteration000rrCDVx}). Then for any $t\geq 0$,
\begin{align}
&\mathbb{E} \left[ \mathcal{E} (\hat{\varphi}_{t + 1}) \right]
\leq \left\| \left[ \prod^t_{k=1}(I - \eta_k \mathscr{L}_K) \right] L_C^{1/2}
\beta^* \right\|_2^2 \nonumber\\
&+\sum^{t}_{k = 1}\eta_k^2 \left( \sigma^2 + \mathbb{E} \sqrt{\mathbb{E}_{x_k}
\langle \beta^*-\hat{\beta}_k, x_k \rangle_2^4}\right)
\left[ \mathbb{E}\left\| \left[ \prod^t_{j = k+1}
(I - \eta_j \mathscr{L}_K) \right] L_C^{1/2} L_K x_k \right\|_2^4 \right]^{1/2},
\label{eqnEDLCase2000ZykHEW}
\end{align}
where the sum of an empty set is defined as zero.
\end{prop}

\begin{proof}
The case $t=0$ is trivial and we assume $t\geq 1$. For any $k$,
\begin{equation*}
L_C^{1/2}(I-\eta_k L_KL_C)=L_C^{1/2}-\eta_k \mathscr{L}_KL_C^{1/2}
=\left(I-\eta_k \mathscr{L}_K\right)L_C^{1/2}.
\end{equation*}
From (\ref{eqn: error decomposition}),
\begin{equation*}
L_C^{1/2}(\hat{\beta}_{t+1}-\beta^*)=-\left[\prod_{k=1}^t\left(I-\eta_k
\mathscr{L}_K\right) \right] L_C^{1/2}\beta^*+
\sum_{k=1}^t \left[ \prod_{j=k+1}^t\left(I-\eta_j \mathscr{L}_K\right) \right]
L_C^{1/2}\mathcal{B}_k.
\end{equation*}

It follows from (\ref{excessGEbyNorm000itb0ac}) that
\begin{align}
&\bE \left[\E (\hat{\varphi}_{t+1})\right]=\bE
\left[\left\| L_C^{1/2}(\hat{\beta}_{t+1}-\beta^*)\right\|_{2}^2\right]
\nonumber \\
=&\bE \left[\left\|-\left[\prod_{k=1}^t\left(I-\eta_k \mathscr{L}_K\right)
\right] L_C^{1/2}\beta^*+
\sum_{k=1}^t \left[ \prod_{j=k+1}^t\left(I-\eta_j \mathscr{L}_K\right) \right]
L_C^{1/2}\mathcal{B}_k\right\|_{2}^2\right] \nonumber \\
=&\left\| \left[ \prod_{k=1}^t(I-\eta_k \mathscr{L}_K) \right]
L_C^{1/2}\beta^*\right\|_{2}^2+
\bE \left[\left\|\sum_{k=1}^t \left[ \prod_{j=k+1}^t(I-\eta_j \mathscr{L}_K)
\right] L_C^{1/2}\mathcal{B}_k\right\|_{2}^2\right]\nonumber \\
=&:\Upsilon^\mathsf{E}_1 + \Upsilon^\mathsf{E}_2
\label{errDecompProp100sbQQck}
\end{align}
where in the expansion of the squared norm, the cross terms vanish
because $\mathbb{E}[\mathcal{B}_k] =
\mathbb{E}_{z_k} [\mathcal{B}_k] = 0$. The notations $\Upsilon^\mathsf{E}_1$
and $\Upsilon^\mathsf{E}_2$ are used only within this proof. When
$k\neq s$, without loss of generality assume $k>s$. Recall that then
$\mathcal{B}_s$ is independent of $z_t$. So,
\begin{gather*}
\mathbb{E}_{z_k} \left< \left[ \prod^t_{j = k+1} (I - \eta_j
\mathscr{L}_K) \right] L_C^{1/2} \mathcal{B}_k, \left[ \prod^t_{j = s+1} (I -
\eta_j \mathscr{L}_K) \right] L_C^{1/2} \mathcal{B}_s \right>_2 = 0.
\end{gather*}
So we expand the squared norm in $\Upsilon^\mathsf{E}_2$,
\begin{align}
\Upsilon^\mathsf{E}_2&=\sum_{k=1}^t\bE \left[\left\|
\left[ \prod_{j=k+1}^t(I-\eta_j \mathscr{L}_K)
\right] L_C^{1/2}\mathcal{B}_k\right\|_{2}^2\right] \nonumber \\
&\leq \sum_{k=1}^t\eta_k^2\bE \left[
\mathbb{E}_{\varepsilon_k}[ (y_k-\langle\hat{\beta}_k,x_k\rangle_2)^2]
\left\| \left[ \prod_{j=k+1}^t(I-\eta_j \mathscr{L}_K) \right] L_C^{1/2} L_K
x_k\right\|_{2}^2 \right], \label{zeroMeanTrick000tCqKUj}
\end{align}
where the inequality is obtained by recalling the zero-mean structure
$\mathcal{B}_t=\eta_t(y_t - \langle \hat{\beta}, x_t \rangle_2)L_Kx_t -
\mathbb{E}_{x_t}[\eta_t(y_t - \langle \hat{\beta}, x_t \rangle_2)L_Kx_t]$
explained in (\ref{B_meanZero000uQjHu}).
Furthermore, we recall that
$\mathbb{E}_{\varepsilon_k}[ (y_k-\langle\hat{\beta}_k,x_k\rangle_2)^2]=
\sigma^2+\langle \beta^*-\hat{\beta}_k,x_k \rangle_2^2$ and obtain
\begin{align}
\Upsilon^\mathsf{E}_2\leq \,& \sigma^2\sum_{k=1}^t\eta_k^2\bE \left[
\left\| \left[ \prod_{j=k+1}^t(I-\eta_j \mathscr{L}_K) \right] L_C^{1/2} L_K
x_k\right\|_{2}^2 \right] \nonumber\\
&+\sum_{k=1}^t\eta_k^2\bE \left[
\langle \beta^*-\hat{\beta}_k,x_k \rangle_2^2
\left\| \left[ \prod_{j=k+1}^t(I-\eta_j \mathscr{L}_K) \right] L_C^{1/2} L_K
x_k\right\|_{2}^2 \right].
\label{unifErrBound1000wU395x}
\end{align}
The proof is complete by applying Cauchy-Schwarz inequality to the right-hand
side of (\ref{unifErrBound1000wU395x}).
\end{proof}

Now we consider the error decomposition for estimation error.

\begin{prop}\label{propBoundHKNorm0005weUo}
Let $\{\hat{\beta}_t\}$ be defined by (\ref{defIteration000rrCDVx}). Assume
$\beta^*\in\mathcal{H}_K$. We have the following error decomposition for any
$t\geq 0$.
\begin{align}
\mathbb{E} & \left[ \| \hat{\beta}_{t+1} - \beta^* \|_K^2 \right] \leq \left\|
\left[ \prod_{k=1}^t(I-\eta_k L_KL_C) \right] \beta^*\right\|_K^2 \nonumber \\
& + \sum^t_{k = 1} \eta_k^2 \left( \sigma^2 + \mathbb{E} \sqrt{\mathbb{E}_{x_k}
\langle \beta^* - \hat{\beta}_k, x_k \rangle_2^4} \right)
\left( \mathbb{E} \left\| \left[ \prod^t_{j = k+1} (I - \eta_jL_KL_C)
\right] L_K x_k \right\|_K^4 \right)^{1/2}
\label{prop: error decomposition in HK with decreasing step size}
\end{align}
\end{prop}
\proof
The proof parallels that of Proposition \ref{prop: error
decomposition in L2 decreasing step size}. By (\ref{eqn: error decomposition})
and the fact $\mathbb{E}[\mathcal{B}_k]=0$ we have
\begin{align}
&\bE \left[\| \hat{\beta}_{t+1}-\beta^*\|_K^2\right]=\bE \left[\left\|-\left[
\prod_{k=1}^t(I-\eta_tL_KL_C)\right]\beta^*+
\sum_{k=1}^t \left[ \prod_{j=k+1}^t(I-\eta_j L_KL_C) \right]
\mathcal{B}_k\right\|_K^2\right] \nonumber\\
&=\left\| \left[ \prod_{k=1}^t(I-\eta_k L_KL_C) \right] \beta^*\right\|_K^2+
\bE \left[\left\|\sum_{k=1}^t \left[ \prod_{j=k+1}^t(I-\eta_j L_KL_C) \right]
\mathcal{B}_k\right\|_K^2\right],
\label{eqn: the first step of error decomposition in HK}
\end{align}
where the second term on the right-hand side is further estimated with the
trick we used in (\ref{zeroMeanTrick000tCqKUj}).
\begin{align*}
& \bE \left[\left\|\sum_{k=1}^t \left[ \prod_{j=k+1}^t(I-\eta_j L_KL_C) \right]
\mathcal{B}_k\right\|_K^2\right]
=\sum_{k=1}^t\mathbb{E}\left[ \left\|
\left[ \prod_{j=k+1}^t(I-\eta_j L_KL_C) \right]
\mathcal{B}_k
\right\|_K^2 \right]\\
\leq & \sum^t_{k = 1} \eta_k^2 \mathbb{E} \left[ (y_k - \langle \hat{\beta}_k,
x_k \rangle_2)^2 \left\| \left[ \prod^t_{j = k+1} (I - \eta_j L_KL_C) \right]
L_Kx_k \right\|_K^2 \right]\\
\leq & \sum^t_{k = 1} \eta_k^2 \left( \sigma^2 + \mathbb{E}
\sqrt{\mathbb{E}_{x_k} \langle \beta^* - \hat{\beta}_k, x_k \rangle_2^4} \right)
\left( \mathbb{E} \left\| \left[ \prod^t_{j = k+1} (I - \eta_jL_KL_C)
\right] L_K x_k \right\|_K^4 \right)^{1/2}.
\end{align*}
The proof is complete.
\qed

\section{Bounding the Excess Generalization Error}\label{section: bounding the excess generalization error}

In this section, we study the excess generalization error
$\mathcal{E}[\hat{\varphi}]$ and prove Theorems \ref{thm: decreasing step size
capacity dependent L2} and \ref{thm: convergence rate in L2 constant step size
and capacity independent}. This is achieved by first estimating the expected
error $\mathbb{E}[\mathcal{E}(\hat{\varphi}_{t+1})]$ for general step-sizes in
Theorem \ref{excessErrorGeneralTheorem000zzlJ7M}, and then substitute specific
settings of step-sizes into the obtained bound.

\subsection{Analysis with General Step-sizes}\label{subsection: analysis with general step-sizes}

In this subsection, we study the excess generalization error with minimal
assumptions on the step-sizes. The following Lemma
\ref{lem: elementary lemma} is a typical application of the spectral theorem on
the polynomial $u^\alpha \prod^b_{j=a}(1-\eta_j u)$ for $u\geq 0$. For a
detailed proof, see e.g.\ \cite[Lemma 2]{ChenTangFanGuo2022-MR4388513}.
Note that when $b<a$, the sum $\sum^b_{j=a}\eta_j$ is defined to be zero.

\begin{lem}\label{lem: elementary lemma}
Let $A$ be a compact positive semi-definite operator on a Hilbert space.
Let $\{ \eta_i \}\subset (0, 1/\|A\|_{\mathsf{op}}]$.
Then for any $a\leq b$ and $\alpha>0$, we have
\begin{gather}\label{eqn: part 2 of the elementary lemma}
\left\| A^\alpha\prod^b_{j=a}(I - \eta_j A)
\right\|^2_{\mathsf{op}} \leq \frac{(\alpha/e)^{2\alpha} +
\|A\|^{2\alpha}_{\mathsf{op}}}{1 + \left( \sum^b_{j=a} \eta_j
\right)^{2\alpha}}.
\end{gather}
When $\alpha = 0$, we have
\begin{gather}\label{prodNormBoundAlphaEqZero000u5QQwy}
\left\| \prod^b_{j=a}(I - \eta_j A) \right\|^2_{\mathsf{op}}\leq 1.
\end{gather}
In particular, when $a>b$, recall that the product
$\prod^b_{j=a}(I - \eta_j A)$ is the identity operator.
So the above estimates
(\ref{eqn: part 2 of the elementary lemma}) and
(\ref{prodNormBoundAlphaEqZero000u5QQwy}) still hold true.\qed
\end{lem}

\begin{lem}\label{compactOperatorBd000R9EK3GO}
Let $X$ be the random function in Model (\ref{model0002LJmHH}). Let $W$ be a
compact operator on $L^2(\mathcal{T})$. We take Assumption \ref{assumption3} to
obtain
\begin{gather*}
\mathbb{E}\left[ \left\| WX \right\|_2^4 \right]\leq c_{\mathsf{M}} \left(
\mathbb{E} \left[ \left\| WX \right\|_2^2 \right] \right)^2.
\end{gather*}
\end{lem}
\proof
Write $W'$ the adjoint
operator of $W$. Then $W'W$ is a compact positive
operator. So we write $\mu_1\geq \mu_2\geq \cdots > 0$ as all the positive
eigenvalues of $W'W$, counting multiplicity. We use an
orthonormal sequence $\{ \psi_i \}$ in $L^2(\mathcal{T})$ as the corresponding
eigenvectors. Assumption \ref{assumption3} implies that
\begin{align*}
\mathbb{E}\left[\|WX\|_2^4 \right] &= \mathbb{E}
\left[\langle X, W'W X \rangle^2_2 \right]\\
&=\mathbb{E}\left[ \left( \sum_i \mu_i\langle \psi_i, X \rangle_2^2
\right)^2\right] = \sum_{i,j} \mu_i\mu_j \mathbb{E}\left[ \langle \psi_i, X
\rangle_2^2 \langle \psi_j, X \rangle_2^2 \right]\\
&\leq \sum_{i,j} \mu_i\mu_j \sqrt{\mathbb{E}[\langle \psi_i, X \rangle_2^4]}
\sqrt{\mathbb{E}[\langle \psi_j, X \rangle_2^4]}
\leq c_{\mathsf{M}}\left( \mathbb{E} \sum_i \mu_i \langle \psi_i, X \rangle_2^2
\right)^2\\
&=c_{\mathsf{M}} \left( \mathbb{E}\left[\|W X\|_2^2 \right]\right)^2.
\end{align*} The proof is then completed.
\qed

\begin{thm}
\label{excessErrorGeneralTheorem000zzlJ7M}
Let $\{ \hat{\beta}_t \}$ be defined by (\ref{defIteration000rrCDVx}) with
step-sizes $\{ \eta_t \} \subset (0, \kappa^{-2}]$. We take Assumption
\ref{assumption1} (with $r>0$), Assumption \ref{assumption2} (with $0<s\leq 1$),
and Assumption \ref{assumption3}. Then for any $t\geq 0$,
\begin{align}
\mathbb{E}[\mathcal{E}(\hat{\varphi}_{t +1})]
\leq&\frac{\|g^*\|_2^2\left( (r/e)^{2r} +\kappa^{4r} \right)}{1 + \left(
\sum^t_{k=1} \eta_k \right)^{2r}} \nonumber \\
&+ \sum^t_{k=1} \eta_k^2 \left( \sigma^2+\sqrt{c_{\mathsf{M}}}
\mathbb{E}[\mathcal{E}
(\hat{\varphi}_{k})] \right)\frac{\sqrt{c_{\mathsf{M}}}
\mathrm{Tr}(\mathscr{L}_K^s)
\left[ (\frac{2-s}{2e})^{2-s} +\kappa^{4-2s}\right]}{1 + \left( \sum^t_{j=k+1}
\eta_j \right)^{2-s}}.
\label{generalExcessErrorBound000DtU6sC}
\end{align}
\end{thm}

\proof
When $t=0$, Bound (\ref{generalExcessErrorBound000DtU6sC}) is reduced to
\begin{gather}\label{tempBound000KmLaN}
\mathbb{E}[\mathcal{E} (0)]\leq \|g^*\|_2^2 ((r/e)^{2r} + \kappa^{4r}).
\end{gather}
We use Assumption \ref{assumption1} to have $\mathcal{E}(0)=\|L_C^{1/2}
\beta^*\|_2^2 = \|\mathscr{L}_K^r g^*\|_2^2\leq \kappa^{4r}\|g^*\|_2^2$,
which verifies (\ref{tempBound000KmLaN}).

Now we assume $t\geq 1$.
We let $J_1$ and $J_2$ denote the two terms in the right-hand side of
(\ref{eqnEDLCase2000ZykHEW}) in Proposition \ref{prop: error decomposition in
L2 decreasing step size}, respectively. That is, $\mathbb{E}[\mathcal{E}
(\hat{\varphi}_{t+1})]\leq J_1 + J_2$, with
\begin{align*}
J_1&= \left\| \left[ \prod^t_{k=1}(I - \eta_k \mathscr{L}_K) \right] L_C^{1/2}
\beta^* \right\|_2^2,\mbox{ and}\\
J_2&=\sum^{t}_{k = 1}\eta_k^2 \left( \sigma^2 + \mathbb{E} \sqrt{\mathbb{E}_{x_k}
\langle \beta^*-\hat{\beta}_k, x_k \rangle_2^4}\right)
\left[ \mathbb{E}\left\| \left[ \prod^t_{j = k+1}
(I - \eta_j \mathscr{L}_K) \right] L_C^{1/2} L_K x_k \right\|_2^4 \right]^{1/2}.
\end{align*}

Assumption \ref{assumption1} gives
$L_C^{1/2}\beta^*=\mathscr{L}_K^r g^*$ for some $r>0$. Recall the assumption
$\{ \eta_j \} \subset (0, \kappa^{-2}]$. We
apply Lemma \ref{lem: elementary lemma} to bound $J_1$,
\begin{gather*}
J_1 = \left\| \left[ \prod^t_{k=1} (I - \eta_k \mathscr{L}_K)
\right]\mathscr{L}_K^r g^* \right\|_2^2
\leq \left\| g^* \right\|_2^2 \frac{(r/e)^{2r} + \kappa^{4r}}{1 + \left(
\sum_{k=1}^t \eta_k \right)^{2r}}.
\end{gather*}
To bound $J_2$, we apply Assumption \ref{assumption3} (the moment condition),
\begin{gather}\label{thm7ToUse1000YQFqz}
\mathbb{E}\sqrt{ \mathbb{E}_{x_k} \langle \beta^*-\hat{\beta}_k, x_k
\rangle_2^4} \leq \sqrt{c_{\mathsf{M}}} \mathbb{E}\left[ \langle \beta^*-
\hat{\beta}_k, x_k \rangle_2^2 \right]
=\sqrt{c_{\mathsf{M}}}\mathbb{E}[\mathcal{E}(\hat{\varphi}_{k})].
\end{gather}
Recall that for any bounded linear operator $A$ on $L^2(\mathcal{T})$,
$\mathbb{E}[\|Ax_t\|_2^2]=\mathbb{E}\mathrm{Tr}(Ax_t\otimes x_t A')=
\mathrm{Tr}(AL_CA')$. We apply Lemma \ref{compactOperatorBd000R9EK3GO},
Assumption \ref{assumption2} (with $0<s\leq 1)$, and Lemma
\ref{lem: elementary lemma} to obtain that
\begin{align}
& \left[ \mathbb{E} \left\| \left[ \prod^t_{j = k+1} (I - \eta_j
\mathscr{L}_K) \right] L_C^{1/2} L_K x_k\right\|_2^4 \right]^{1/2} \nonumber\\
\leq & \sqrt{c_{\mathsf{M}}}\mathbb{E} \left\| \left[ \prod^t_{j = k+1} (I -
\eta_j \mathscr{L}_K) \right] L_C^{1/2} L_K x_k \right\|_2^2
= \sqrt{c_{\mathsf{M}}} \mathrm{Tr} \left( \mathscr{L}_K^2 \prod^t_{j = k+1}
(I - \eta_j \mathscr{L}_K)^2 \right) \nonumber\\
\leq & \sqrt{c_{\mathsf{M}}} \mathrm{Tr}(\mathscr{L}_K^s) \left\|
\mathscr{L}_K^{1 - \frac{s}2}
\prod^t_{j = k+1}(I - \eta_j\mathscr{L}_K) \right\|_{\mathsf{op}(L^2)}^2
\leq \sqrt{c_{\mathsf{M}}} \mathrm{Tr}(\mathscr{L}_K^s) \frac{\left(
\frac{2-s}{2e} \right)^{2-s}
+\kappa^{4-2s}}{1+\left( \sum^t_{j=k+1}\eta_j \right)^{2-s}}.
\label{thm7ToUse2000EwgxHt}
\end{align}
The proof is complete.
\qed

\begin{prop}
\label{coarseEst000XbG0PR}
Let $t\geq 1$. Let $\{ \hat{\beta}_k \}$ be defined by
(\ref{defIteration000rrCDVx}) with step-sizes $\{ \eta_k \}\subset (0,
\kappa^{-2}]$. We take Assumption \ref{assumption2} with $0 < s \leq 1$
(in particular, Assumption \ref{assumption2} is not needed when $s=1$) and
Assumption \ref{assumption3}. In particular, when $t\geq 2$ we assume for
any $k\leq t-1$ that
\begin{gather}\label{firstStepCond000s5c7AU}
c_{\mathsf{M}}\mathrm{Tr}(\mathscr{L}_K^s) \left[ \left( \frac{2-s}{2e}
\right)^{2-s} + \kappa^{4-2s} \right]
\sum^k_{l = 1} \frac{\eta_l^2}{1 + \left( \sum^k_{j = l + 1} \eta_j
\right)^{2-s}} \leq \frac12.
\end{gather}
Then we have a coarse estimation of the expected excess generalization error
for $k=1, \cdots, t$,
\begin{gather}
\label{boundCoarseExcessError000DVlpcS}
\mathbb{E}\left[ \mathcal{E} (\hat{\varphi}_{k}) \right] \leq 2\left\| \beta^*
\right\|_2^2 + \frac{\sigma^2}{\sqrt{c_{\mathsf{M}}}}.
\end{gather}
\end{prop}

We see that (\ref{boundCoarseExcessError000DVlpcS}) only provides a coarse
bound $\mathbb{E}\left[ \mathcal{E} (\hat{\varphi}_{k}) \right] = O(1)$.
However, the designed purpose of Proposition \ref{coarseEst000XbG0PR} is to
estimate $\mathbb{E}\left[ \mathcal{E} (\hat{\varphi}_{k}) \right]$ in the
right-hand side of (\ref{generalExcessErrorBound000DtU6sC}) in our convergence
analysis, and a bound finer than $O(1)$ would not serve the purpose better,
because a constant variance $\sigma^2$ is added to $\sqrt{c_{\mathsf{M}}}
\mathbb{E} [\mathcal{E} (\hat{\varphi}_k)]$ in
(\ref{generalExcessErrorBound000DtU6sC}).

\proof[Proof of Proposition \ref{coarseEst000XbG0PR}]

We organize the proof by induction. Recall that $\hat{\varphi}_1=0$ and
$\|L_C\|_{\mathsf{op}(L^2)}\leq 1$. When $t=1$, (\ref{firstStepCond000s5c7AU})
is verified by
\begin{gather*}
\mathcal{E}(0) = \|L_C^{1/2}\beta^*\|_2^2 \leq \|\beta^*\|_2^2.
\end{gather*}
Let $T\geq 2$, and we assume Proposition \ref{coarseEst000XbG0PR} holds for
$t=1, \ldots, T-1$. To finish the induction, we need only to prove Proposition
\ref{coarseEst000XbG0PR} for $t = T$. That is, we assume
(\ref{firstStepCond000s5c7AU}) and (\ref{boundCoarseExcessError000DVlpcS}) for
$t = 1, \ldots, T - 1$ and need only to prove
(\ref{boundCoarseExcessError000DVlpcS}) for $t = T$. To this end, we use
Proposition \ref{prop: error decomposition in L2 decreasing step size} to have
\begin{gather}\label{temp000KvOikS}
\mathbb{E}[\mathcal{E}(\hat{\varphi}_T)]\leq \Upsilon^T_1 +
\sum^{T - 1}_{k = 1} \eta_k^2 \left( \sigma^2 + \mathbb{E}
\sqrt{\mathbb{E}_{x_k} \langle \beta^* - \hat{\beta}_k, x_k \rangle_2^4}
\right) \Upsilon^T_{2, k},
\end{gather}
where
\begin{align*}
\Upsilon^T_1&=\left\| \left[ \prod^{T-1}_{k=1}(I - \eta_k\mathscr{L}_K) \right]
L_C^{1/2}\beta^* \right\|_2^2,\quad\mbox{and}\\
\Upsilon^T_{2,k} &= \left[ \mathbb{E}\left\| \left[ \prod^{T-1}_{j=k+1}
(I - \eta_j \mathscr{L}_K) \right] L_C^{1/2} L_K x_k \right\|^4_2
\right]^{1/2},\quad k=1,\ldots,T - 1.
\end{align*}

To bound $\Upsilon^T_1$, we note that $\|I-\eta_k\mathscr{L}_K\|_{\mathsf{op}
(L^2)}\leq 1$ because $\eta_k\in(0, \kappa^{-2}]$ and $\|\mathscr{L}_K
\|_{\mathsf{op}(L^2)}\leq \kappa^2$. So,
\begin{gather*}
\Upsilon^T_1\leq \|L_C^{1/2} \beta^*\|_2^2\leq \|\beta^*\|_2^2.
\end{gather*}

Next, we follow (\ref{thm7ToUse1000YQFqz}), use the induction assumption and
Assumption \ref{assumption3} (the moment condition) to obtain
\begin{gather*}
\mathbb{E}\sqrt{ \mathbb{E}_{x_k} \langle \beta^* - \hat{\beta}_k, x_k
\rangle_2^4 } \leq \sqrt{c_{\mathsf{M}}} \mathbb{E}[\mathcal{E}
(\hat{\varphi}_k)] \leq
\sigma^2+2\sqrt{c_{\mathsf{M}}} \|\beta^*\|_2^2, \quad k=1, \ldots, T - 1.
\end{gather*}
Then, we follow (\ref{thm7ToUse2000EwgxHt}) and use Assumption 2 with
$0<s\leq 1$ to obtain
\begin{align*}
\Upsilon^T_{2, k} & \leq \sqrt{c_{\mathsf{M}}} \mathbb{E} \left\| \left[
\prod^{T-1}_{j=k+1} (I - \eta_j \mathscr{L}_K) \right] L_C^{1/2}L_K x_k
\right\|_2^2\\
&\leq \sqrt{c_{\mathsf{M}}} \mathrm{Tr}(\mathscr{L}_K^s)\frac{\left(
\frac{2-s}{2e} \right)^{2-s} + \kappa^{4-2s}}{1 + (\sum^{T - 1}_{j=k+1}
\eta_j)^{2-s}}, \quad k=1, \ldots, T - 1.
\end{align*}
We continue (\ref{temp000KvOikS}) and use Condition
(\ref{firstStepCond000s5c7AU}) for $k=1,\ldots,T - 1$ to have
\begin{align*}
\mathbb{E}[\mathcal{E}(\hat{\varphi}_T)] &\leq \|\beta^*\|_2^2 +
\sum^{T-1}_{k=1} \eta^2_k (2\sigma^2+2\sqrt{c_{\mathsf{M}}} \|\beta^*\|_2^2)
\sqrt{c_{\mathsf{M}}}
\mathrm{Tr}(\mathscr{L}_K^s) \frac{\left( \frac{2-s}{2e} \right)^{2-s} +
\kappa^{4-2s}}{1 + (\sum^{T - 1}_{j=k+1}\eta_j)^{2-s}}\\
&\leq \|\beta^*\|_2^2 + \frac{\sigma^2}{\sqrt{c_{\mathsf{M}}}} +
\|\beta^*\|_2^2.
\end{align*}
This completes the proof.\qed

\subsection{Analysis in Online and Finite-horizon Settings of
Step-sizes}\label{subsection: analysis in online and finite-horizon settings of step-sizes}

In this subsection we study the excess generalization error in the online
and finite-horizon settings of step-sizes, respectively. The following Lemma
\ref{lemOnlineStepSize000m7uT2} is commonly used in the literature
\cite{ChenTangFanGuo2022-MR4388513, YingPontil2008-MR2443089,
guo2022rates} with smaller ranges of parameters $\theta$ and $\nu$.
In this paper, we need coverage of the whole domain $\nu>0$ and $0<\theta<1$, and the
proof is not elsewhere available to our best knowledge.

\begin{lem}\label{lemOnlineStepSize000m7uT2}
Let $t \geq 1$, $\eta_k = \eta_0 k^{-\theta}$ with $\eta_0>0$ and $0<\theta<1$.
For any $\nu>0$,
\begin{gather}
\label{boundOnlineStepSize1}
\left( \sum^t_{k = 1}\eta_k \right)^{-\nu} \leq \left( \frac{\eta_0 (1 -
2^{\theta-1})}{1-\theta} \right)^{-\nu} (t+1)^{-\nu(1-\theta)},
\end{gather}
and
\begin{gather*}
\sum^t_{k=1}\frac{\eta_k^2}{1+\left( \sum^t_{j=k+1} \eta_j \right)^{\nu}} \leq
C^{\mathsf{OL}}\left\{\begin{array}{ll}
(t+1)^\omega\log(t+1),& (\nu,\theta)\in\Omega,\\
(t+1)^\omega, & (\nu,\theta)\not\in\Omega,
\end{array}\right.
\end{gather*}
where $\Omega = \{ (\nu,\theta): 0<\nu\leq 1\mbox{ and }\theta=1/2 \}
\cup \{ (\nu,\theta): \nu=1\mbox{ and }0<\theta\leq 1/2 \}$,
$C^{\mathsf{OL}}$ is a constant independent of $t$, and
\begin{gather}\label{omegaDefCopied000uDiM}
\omega=\omega(\nu,\theta)=\left\{\begin{array}{ll}
1-2\theta-\nu+\nu\theta,&0<\nu\leq 1\mbox{ and }0<\theta\leq 1/2,\\
-\theta,& \nu\geq 1\mbox{ and }0<\theta\leq \nu/(\nu+1),\\
-\nu(1-\theta),& 1/2\leq\theta<1 \mbox{ and }\theta \geq \nu/(\nu+1).
\end{array}\right.
\end{gather}
In particular, when $\nu\geq1$, $\omega=-\min\{ \theta,\nu(1-\theta) \}$.
The constant $C^{\mathsf{OL}}$ will be specified by
(\ref{COLDef000U9snA}) in the proof.
\end{lem}

Lemma \ref{lemOnlineStepSize000m7uT2} is based on Lemma
\ref{techLemma000avyDpd} in Appendix.
Same as Lemma \ref{techLemma000avyDpd}, we purposely allow the domains to overlap
in (\ref{omegaDefCopied000uDiM}), to simplify the usage later.
We will elucidate the parameter $\omega$
and the set $\Omega$ by Figure \ref{figOmega0006Ne1p} in Appendix.

\proof[Proof of Lemma \ref{lemOnlineStepSize000m7uT2}]
For any $k\geq 0$ and $t\geq k+1$,
\begin{gather}
\label{estTemp000T2LG4wa}
\sum^t_{j=k+1} \eta_j=\eta_0\sum^t_{j=k+1} j^{-\theta}\geq
\frac{\eta_0}{1-\theta} \left[ (t+1)^{1-\theta}
- (k+1)^{1-\theta} \right].
\end{gather}
We set $k=0$ to have
\begin{gather}\label{estTemp1000Z7aTU}
\sum^t_{j=1} \eta_j \geq \frac{\eta_0
(1-2^{\theta-1})}{1-\theta} (t+1)^{1-\theta}.
\end{gather}
One raises (\ref{estTemp1000Z7aTU}) to power $-\nu$ to obtain
(\ref{boundOnlineStepSize1}).

Recall that for $k\geq 1$, one has $k\geq (k + 2)/3$, and
\begin{gather*}
\sum^t_{k = 1}\frac{\eta_k^2}{1 + \left( \sum^t_{j=k+1}\eta_j \right)^\nu}\leq
\eta_t^2+\sum^{t - 1}_{k = 1}\frac{\eta_0^2 (k+2)^{-2\theta} 3^{2\theta}}{1 +
\left( \frac{\eta_0}{1-\theta} \right)^\nu\left[ (t+1)^{1-\theta} -
(k+1)^{1-\theta} \right]^\nu}.
\end{gather*}
Note that for any $k=1,\ldots,t-1$,
\begin{gather*}
\frac{(k+2)^{-2\theta}}{1+\left[ (t+1)^{1-\theta}-(k+1)^{1-\theta} \right]^\nu}
\leq\int^{k+2}_{k+1} \frac{u^{-2\theta}du}{1+[(t+1)^{1-\theta} -
u^{1-\theta}]^\nu}.
\end{gather*}
We use Lemma \ref{techLemma000avyDpd} to have
\begin{gather*}
\sum^t_{k = 1}\frac{\eta_k^2}{1 + \left( \sum^t_{j=k+1}\eta_j \right)^\nu}
\leq \eta_0^2 t^{-2\theta} +\frac{3^{2\theta}\eta_0^2}{\min\left\{ 1,
\left( \frac{\eta_0}{1-\theta} \right)^\nu \right\}} \int^{t + 1}_2
\frac{u^{-2\theta} du}{1 + [(t+1)^{1 - \theta} - u^{1-\theta}]^\nu}\\
\leq\eta_0^2 2^{2\theta} (t+1)^{-2\theta} +
\frac{3^{2\theta}\eta_0^2 C_0^{\mathsf{OL}}}{\min\left\{ 1,
\left( \frac{\eta_0}{1-\theta} \right)^\nu \right\}}\times
\left\{\begin{array}{ll}
(t+1)^\omega\log(t+1),& (\nu,\theta) \in \Omega,\\
(t+1)^\omega,& (\nu,\theta) \not\in \Omega.
\end{array}\right.
\end{gather*}
Now we verify that $-2\theta\leq \omega$ on the whole domain $(0, \infty)\times
(0,1)$ of parameters. When $0<\theta\leq 1/2$ and $0<\nu\leq 1$,
$\omega = -2\theta+(1-\nu)+\nu\theta\geq -2\theta$. When $1/2<\theta<1$
and $0<\nu\leq 1$, $\omega=-\nu(1-\theta)\geq -1+\theta>-\theta>-2\theta$.
When $0<\theta<\nu/(\nu+1)$ and $\nu>1$, $\omega=-\theta>-2\theta$.
When $\nu/(\nu+1)\leq\theta<1$ and $\nu>1$, $\theta>\nu/(\nu+2)$, so
$\omega=-\nu(1-\theta)>-2\theta$.

We complete the proof by letting
\begin{gather}\label{COLDef000U9snA}
C^{\mathsf{OL}} = \frac{\eta_0^2 2^{2\theta}}{\log 2}
+\frac{3^{2\theta} \eta_0^2 C_0^{\mathsf{OL}}}{\min
\left\{ 1,\left( \frac{\eta_0}{1-\theta} \right)^\nu \right\}}.
\end{gather}
\qed

\proof[Proof of Theorem \ref{thm: decreasing step size capacity dependent L2}]
First, we shall apply Proposition \ref{coarseEst000XbG0PR}. To verify the
assumptions in Proposition \ref{coarseEst000XbG0PR},
we need only to determine the constant $C^{\mathsf{S}}_1$
to guarantee (\ref{firstStepCond000s5c7AU}), i.e., for $k=1, \ldots, t-1$,
\begin{gather}\label{firstStepCOPY000Tg8d}
c_{\mathsf{M}}\mathrm{Tr}(\mathscr{L}_K^s) \left[ \left( \frac{2-s}{2e}
\right)^{2-s} + \kappa^{4-2s} \right]
\sum^k_{l = 1} \frac{\eta_l^2}{1 + \left( \sum^k_{j = l + 1} \eta_j
\right)^{2-s}} \leq \frac12.
\end{gather}
Recall $r>0$ and $0<s\leq1$. We apply Lemma \ref{lemOnlineStepSize000m7uT2} with
\begin{gather}\label{tempCond000MCl2}
\nu = 2-s\geq1,\quad \mbox{and,}\quad 0<\theta=\frac{\min\{ 2r, 2-s \}}{1+
\min\{ 2r, 2-s \}}\leq \frac{\nu}{\nu+1},
\end{gather}
so $\omega = -\theta<0$, and for $k=1,\ldots,t - 1$,
\begin{gather}\label{temp0001VrJmP}
\sum^k_{l=1}\frac{\eta_l^2}{1+\left( \sum^k_{j=l+1}\eta_j \right)^{2-s}}
\leq C^{\mathsf{OL}}\left\{\begin{array}{ll}
(k+1)^{-\theta}\log(k+1),& s=1,\\
(k+1)^{-\theta},& 0<s<1.
\end{array}\right.
\end{gather}
Recall $\eta_0\leq1$. The above inequality is continued by
\begin{align}
C^{\mathsf{OL}}&=\frac{\eta_0^2 2^{2\theta}}{\log 2}+
\frac{3^{2\theta}\eta_0^2 C_0^{\mathsf{OL}}}{\min\left\{ 1,\left(
\frac{\eta_0}{1-\theta} \right)^{2-s} \right\}}
\leq \frac{\eta_0^s 4^\theta}{\log 2}+
\frac{9^\theta\eta_0^2 C_0^{\mathsf{OL}}}{\eta_0^{2-s}
\min\{ 1, (1-\theta)^{-2+s} \}} \nonumber \\
&\leq \eta_0^s\left( \frac{4^\theta}{\log 2} + 9^\theta C_0^{\mathsf{OL}} \right).
\label{tempCOLBound000Viqt}
\end{align}
On the other hand, $(k+1)^{-\theta}\leq 1$,
and $(k+1)^{-\theta}\log(k+1)\leq \frac{1}{e\theta}$ (see (\ref{polyBoundingLog000zNwX})).
Therefore, to achieve (\ref{firstStepCOPY000Tg8d})
(which is just (\ref{firstStepCond000s5c7AU}) for Proposition
\ref{coarseEst000XbG0PR}), we need simply to let
\begin{gather}\label{specifyC1S000hVp9yJ}
C_1^{\mathsf{S}} = \left\{2c_{\mathsf{M}}\mathrm{Tr}(\mathscr{L}^s_K)
\left[ \left( \frac{2-s}{2e} \right)^{2-s} +\kappa^{4-2s} \right]
\left( \frac{4^\theta}{\log 2} + 9^\theta C_0^{\mathsf{OL}}
\right)
\left( 1+ \frac{1}{e\theta}\right) \right\}^{-1/s}.
\end{gather}

Second, we apply Theorem \ref{excessErrorGeneralTheorem000zzlJ7M},
of which the conditions are now all satisfied.
We plug (\ref{boundCoarseExcessError000DVlpcS}) of Proposition
\ref{coarseEst000XbG0PR}, into (\ref{generalExcessErrorBound000DtU6sC})
of Theorem \ref{excessErrorGeneralTheorem000zzlJ7M}, to obtain
\begin{align}
\mathbb{E}[\mathcal{E}(\hat{\varphi}_{t+1})] \leq &
\frac{\|g^*\|_2^2((r/e)^{2r} + \kappa^{4r})}{\left(
\sum^t_{k=1}\eta_k \right)^{2r}} \nonumber \\
&+\sqrt{c_{\mathsf{M}}} \mathrm{Tr}(\mathscr{L}_K^s)
\left[ \left( \frac{2-s}{2e} \right)^{2-s} +\kappa^{4-2s} \right]
\left( 2\|\beta^*\|_2^2\sqrt{c_{\mathsf{M}}}+2\sigma^2 \right)
\nonumber \\
&\times\sum^t_{k=1} \frac{\eta_k^2}{1+ \left( \sum^t_{j=k+1} \eta_j
\right)^{2-s}}. \label{tempBound1000VULt2}
\end{align}
For the first term in the right-hand side of (\ref{tempBound1000VULt2}), we
apply Lemma \ref{lemOnlineStepSize000m7uT2} with $\nu=2r$.
The last sum in (\ref{tempBound1000VULt2}) is bounded above in (\ref{temp0001VrJmP}).
We have
\begin{align*}
\mathbb{E}[\mathcal{E}(\hat{\varphi}_{t+1})]\leq &
\|g^*\|_2^2 ((r/e)^{2r} + \kappa^{4r}) \left( \frac{\eta_0
(1-2^{\theta-1})}{1-\theta} \right)^{-2r} (t+1)^{-2r(1-\theta)}\\
&+2(\sigma^2+\sqrt{c_{\mathsf{M}}}\|\beta^*\|_2^2) \sqrt{c_{\mathsf{M}}}
\mathrm{Tr}(\mathscr{L}_K^s)
\left[ \left( \frac{2-s}{2e} \right)^{2-s} +\kappa^{4-2s} \right]
C^{\mathsf{OL}}\\
&\times\left\{\begin{array}{ll}
(k+1)^{-\theta}\log(k+1),& s=1,\\
(k+1)^{-\theta},& 0<s<1.
\end{array}\right.
\end{align*}
From $\theta= \frac{\min\{ 2r,\nu \}}{1+\min\{ 2r,\nu
\}}\leq \frac{2r}{1+2r}$, we have $-2r(1-\theta)\leq -\theta$.
Therefore,
\begin{gather*}
\mathbb{E}[\mathcal{E}(\hat{\varphi}_{t+1})]\leq C_1
\left\{\begin{array}{ll}
(k+1)^{-\theta}\log(k+1),& s=1,\\
(k+1)^{-\theta},& 0<s<1.
\end{array}\right.
\end{gather*}
with
\begin{align}
C_1=& \|g^*\|_2^2 \left( (r/e)^{2r}+\kappa^{4r} \right)\left( \frac{\eta_0
(1-2^{\theta-1})}{1-\theta} \right)^{-2r} \nonumber \\
&+2(\sigma^2+\sqrt{c_{\mathsf{M}}}\|\beta^*\|_2^2) \sqrt{c_{\mathsf{M}}}
\mathrm{Tr}(\mathscr{L}_K^s)
\left[ \left( \frac{2-s}{2e} \right)^{2-s} +\kappa^{4-2s} \right]
C^{\mathsf{OL}}.\label{C1Def0001LaZ}
\end{align}
\qed

\proof[Proof of Theorem \ref{thm: convergence rate in L2 constant step size and
capacity independent}]
First, for applying Proposition \ref{coarseEst000XbG0PR}, we need only to find an
upper bound $C^{\mathsf{S}}_2$ of step-sizes to guarantee (\ref{firstStepCond000s5c7AU}),
i.e., for $k=1,\ldots,T-1$,
\begin{gather}\label{EQfromProp11ForThm2000BbuH}
c_{\mathsf{M}}\mathrm{Tr}(\mathscr{L}_K^s) \left[ \left( \frac{2-s}{2e}
\right)^{2-s} + \kappa^{4-2s} \right]
\sum^k_{l = 1} \frac{\eta_l^2}{1 + \left( \sum^k_{j = l + 1} \eta_j
\right)^{2-s}} \leq \frac12.
\end{gather}
Recall $\eta_t=\eta_0 T^{-2r/(2r+1)}$. We write $\eta=\eta_t$. For any $k\leq T-1$,
we bound the sum in (\ref{EQfromProp11ForThm2000BbuH}) as
\begin{align}
\sum^k_{l = 1} \frac{\eta_l^2}{1 + \left( \sum^k_{j = l + 1} \eta_j
\right)^{2-s}} = \eta^2+\sum^{k - 1}_{t = 1}\frac{\eta^2}{1+(t\eta)^{2-s}}
\leq \eta^2+\eta\int^{k-1}_0 \frac{\eta du}{1+(\eta u)^{2-s}}
\nonumber \\
\leq \eta^2+\eta\int^{\eta T} _0\frac{du}{1+u^{2-s}}
\leq \eta^2+\eta\left\{\begin{array}{ll}
\frac{2-s}{1-s},& 0<s<1,\\
\log(\eta T + 1),& s=1,
\end{array}\right.
\label{temp000FxPypl}
\end{align}
where in the last inequality we used (\ref{intRational000pZZb}) for $0<s<1$.
Recall that $\eta_0\leq 1$.
\begin{gather*}
\eta \log(\eta T+1)\leq \eta\log(T^{\frac{1}{2r+1}} + 1)
\leq T^{\frac{-2r}{2r+1}}\left( \log 2+\frac{1}{2r+1} \log T\right)
\leq 1+\frac{1}{2er},
\end{gather*}
where in the last inequality we used (\ref{polyBoundingLog000zNwX}).
We have a coarse estimate for $k\leq T-1$,
\begin{gather*}
\sum^k_{l = 1} \frac{\eta_l^2}{1 + \left( \sum^k_{j = l + 1} \eta_j
\right)^{2-s}}\leq C^{\mathsf{S}}_{2*}:=2+\left\{\begin{array}{ll}
1/(1-s),& 0<s<1,\\
1/(2er),& s=1.
\end{array}\right.
\end{gather*}
Therefore, to guarantee (\ref{EQfromProp11ForThm2000BbuH}) (which is just
(\ref{firstStepCond000s5c7AU}) for Proposition \ref{coarseEst000XbG0PR}),
we just need
\begin{gather} \label{specifyC2S000zO2SiQ}
C^{\mathsf{S}}_2=\left\{ 2c_{\mathsf{M}}\mathrm{Tr}(\mathscr{L}_K^s)
\left[ \left( \frac{2-s}{2e}
\right)^{2-s} + \kappa^{4-2s} \right]C^{\mathsf{S}}_{2*} \right\}^{-1/s}.
\end{gather}

We plug (\ref{boundCoarseExcessError000DVlpcS}) of Proposition
\ref{coarseEst000XbG0PR}, into (\ref{generalExcessErrorBound000DtU6sC})
of Theorem \ref{excessErrorGeneralTheorem000zzlJ7M} to obtain
\begin{align}
\mathbb{E}[\mathcal{E}(\hat{\varphi}_{T+1})]\leq &
\frac{\|g^*\|_2^2 ((r/e)^{2r} + \kappa^{4r})}{\left( \sum^T_{k=1}\eta_k \right)^{2r}}
+\sqrt{c_{\mathsf{M}}} \mathrm{Tr}(\mathscr{L}_K^s)
\left[ \left( \frac{2-s}{2e} \right)^{2-s} + \kappa^{4-2s} \right]\nonumber \\
&\times (2\|\beta^*\|_2^2 \sqrt{c_{\mathsf{M}}} + 2\sigma^2)
\sum^T_{k=1} \frac{\eta_k^2}{1+\left( \sum^T_{j=k+1} \eta_j \right)^{2-s}}.
\label{temp000ochT}
\end{align}
At the right-hand side, the first term is bounded with
\begin{gather*}
\left( \sum^T_{k=1}\eta_k \right)^{-2r}\leq \eta_0^{-2r} T^{-2r/(2r+1)},
\end{gather*}
and the second term is bounded in (\ref{temp000FxPypl}),
\begin{gather*}
\sum^T_{k=1}\frac{\eta_k^2}{1+\left( \sum^T_{j=k+1} \eta_j \right)^{2-s}}
\leq \eta_0^2 T^{-\frac{2r}{2r+1}}
+\eta_0\left\{\begin{array}{ll}
\frac{2-s}{1-s}T^{-\frac{2r}{2r+1}},& 0<s<1,\\
\frac{2r+2}{2r+1}T^{-\frac{2r}{2r+1}}\log (T+1),& s=1,
\end{array}\right.
\end{gather*}
where in the case $s=1$ we have used
\begin{gather*}
\log(T^a+1)\leq \log 2+a\log T\leq (a+1)\log (T+1),\quad \mbox{for any }T\geq 1,a>0.
\end{gather*}
The proof is therefore completed by letting
\begin{align}
C_2=&\eta_0^{-2r} \|g^*\|_2^2 ((r/e)^{2r} + \kappa^{4r})+2\sqrt{c_{\mathsf{M}}}
\mathrm{Tr}(\mathscr{L}_K^s)
\left[ \left( \frac{2-s}{2e} \right)^{2-s} + \kappa^{4-2s} \right]\nonumber\\
&\times(\|\beta^*\|_2^2 \sqrt{c_{\mathsf{M}}} + \sigma^2)
\times\left\{\begin{array}{ll}
\eta_0^2+\eta_0\frac{2-s}{1-s},&0<s<0,\\
\eta_0^2+\eta_0\frac{2r+2}{2r+1},& s=1.
\end{array}\right. \label{specifyC2000lBOk5F}
\end{align}
\qed

\section{Bounding the Estimation Error}\label{section: bounding the estimation error}

In this section, we bound the estimation error in $\mathcal{H}_K$ metric and
prove Theorems \ref{thm3HKOnline000EGE8Dv} and \ref{thm4HKFH0000KvQ6A}. This is
achieved by first estimating the expected error $\mathbb{E}[\|\hat{\beta}-
\beta^*\|_K^2]$ for general step-sizes in Theorem \ref{boundingEstErr0003fOlO},
and then substitute specific settings of step-sizes into the obtained bound.

\begin{thm}\label{boundingEstErr0003fOlO}
Let $t\geq 0$ be an integer. Let $\{ \hat{\beta}_k: 1\leq k\leq t+1 \}$ be
defined by (\ref{defIteration000rrCDVx}) with step-sizes $\{ \eta_k \}\subset
(0, \kappa^{-2}]$.  We take Assumption \ref{assumption2} (with $0<s\leq 1$),
\ref{assumption3}, and \ref{assumption4} (with $r>0$).  In particular, when
$t\geq 2$ we assume for any $k\leq t-1$ that
\begin{gather}\label{thmHKConvStepSizeCond00056cR0}
c_{\mathsf{M}} \mathrm{Tr}(\mathscr{L}_K^s)\left[ \left( \frac{2-s}{2e}
\right)^{2-s}+ \kappa^{4-2s} \right] \sum^k_{l=1}\frac{\eta_l^2}{1+\left(
\sum^k_{j=l+1} \eta_j \right)^{2 - s}} \leq \frac12.
\end{gather}
Then,
\begin{gather*}
\mathbb{E}[\| \hat{\beta}_{t+1} - \beta^* \|_K^2] \leq C^K\left[ \left(
\sum^t_{k=1}\eta_k \right)^{-2r} +\sum^t_{k=1} \frac{\eta_k^2}{1+\left(
\sum^t_{j=k+1} \eta_j \right)^{1-s}} \right],
\end{gather*}
where $C^K$ is a constant independent of $t$ and will be specified in the
proof, and when $s=1$, $(\sum^t_{j=k+1} \eta_j)^{1-s}:=1$ even when $k=t$ that
vanishes the sum.
\end{thm}

\proof
Assumption \ref{assumption4} that $\beta^*=L_K^{1/2}\mathscr{L}_C^rg^\dag$ (for
some $g^\dag\in L^2(\mathcal{T})$ and $r>0$) guarantees $\beta^*\in
\mathcal{H}_K$. So we start from Proposition \ref{propBoundHKNorm0005weUo} by
bounding
$\mathbb{E}[\|\hat{\beta}_{t+1} - \beta^*\|_K^2]\leq \Upsilon^K_1+
\Upsilon^K_2$, where
\begin{align*}
\Upsilon^K_1 = & \left\| \left[ \prod^t_{k=1}(I - \eta_k L_KL_C) \right]
\beta^* \right\|_K^2,\quad\mbox{and}\\
\Upsilon^K_2 = & \sum^t_{k=1}\eta_k^2 (\sigma^2 + \mathbb{E}\sqrt{
\mathbb{E}_{x_k} \langle \beta^* - \hat{\beta}_k, x_k \rangle_2^4}) \left(
\mathbb{E} \left\| \left[ \prod^t_{j = k+1} (I - \eta_j L_KL_C) \right] L_K x_k
\right\|_K^4 \right)^{1/2}.
\end{align*}

Recall that $\mathscr{L}_C=L_K^{1/2}L_CL_K^{1/2}$. We use Lemma
\ref{lem: elementary lemma} to bound the operator norm,
\begin{align*}
\Upsilon^K_1 & \leq \left\| L_K^{1/2} \left[ \prod^t_{k=1} (I-\eta_k
\mathscr{L}_C) \right] \mathscr{L}_C^r g^\dag \right\|_K^2
\leq \kappa^2 \left\| \mathscr{L}_C^r \prod^t_{k=1}(I - \eta_k \mathscr{L}_C)
\right\|^2_{\mathsf{op}(L^2)} \|g^\dag\|_2^2\\
&\leq \frac{\kappa^2 \|g^\dag\|_2^2 ((r/e)^{2r} + \|\mathscr{L}_C
\|^{2r}_{\mathsf{op}(L^2)})}{1 + (\sum^t_{k=1} \eta_k)^{2r}}.
\end{align*}

For $\Upsilon^K_2$, we consider its different factors separately. First, recall
that $\hat{\beta}_k$ is independent of $x_k$. Assumption \ref{assumption3}
(moment condition) guarantees that
\begin{gather*}
\mathbb{E}\sqrt{\mathbb{E}_{x_k} \langle \beta^* - \hat{\beta}_k, x_k
\rangle_2^4} \leq \sqrt{c_{\mathsf{M}}} \mathbb{E} [\langle \beta^* -
\hat{\beta}_k , x_k\rangle_2^2] = \sqrt{c_{\mathsf{M}}}
\mathbb{E}[\mathcal{E}(\hat{\varphi}_{t})].
\end{gather*}
With Proposition \ref{coarseEst000XbG0PR}, our assumption on step-sizes
guarantees that
\begin{gather*}
\mathbb{E}[\mathcal{E}(\hat{\varphi}_k)] \leq 2\|\beta^*\|_2^2 +
\frac{\sigma^2}{\sqrt{c_{\mathsf{M}}}}, \mbox{ for all }k=1,\ldots, t.
\end{gather*}
Second, we use Lemma \ref{compactOperatorBd000R9EK3GO} and recall that
$\|L_K^{1/2} f\|_K = \|f\|_2$ for any $f\in L^2(\mathcal{T})$ to obtain
\begin{align*}
\Upsilon^K_{2*}&:=\left( \mathbb{E} \left\| \left[ \prod^t_{j=k+1} (I - \eta_j
L_KL_C)\right] L_K x_k \right\|_K^4 \right)^{1/2}\\
&=\left( \mathbb{E} \left\| \left[ \prod^t_{j=k+1} (I-\eta_j \mathscr{L}_C
\right] L_K^{1/2} x_k \right\|_2^4 \right)^{1/2}\\
&\leq \sqrt{c_{\mathsf{M}}} \mathbb{E} \left\| \left[ \prod^t_{j=k+1} (I -
\eta_j\mathscr{L}_C) \right] L_K^{1/2} x_k \right\|_2^2.
\end{align*}
Recall that $\mathbb{E}[\|Ax_t\|_2^2] = \mathbb{E}\mathrm{Tr}(Ax_t\otimes x_t
A') = \mathrm{Tr} (AL_CA')$ for any bounded linear operator $A$ on
$L^2(\mathcal{T})$.
\begin{align*}
\Upsilon^K_{2*}&\leq \sqrt{c_{\mathsf{M}}} \mathrm{Tr}\left(
\mathscr{L}_C\prod^t_{j=k+1} (I - \eta_j \mathscr{L}_C)^2 \right) \\
&\leq \sqrt{c_{\mathsf{M}}} \mathrm{Tr}(\mathscr{L}^s_C)\left\|
\mathscr{L}_C^{\frac{1-s}{2}} \prod^t_{j=k+1} (I - \eta_j \mathscr{L}_C)
\right\|_{\mathsf{op}(L^2)}^2,
\end{align*}
where we abuse the notation a little and let $\mathscr{L}_C^{(1-s)/2}$ denote
the identity operator when $s=1$. Thanks to Lemma \ref{lem: elementary lemma},
\begin{gather*}
\Upsilon^K_{2*}\leq \sqrt{c_{\mathsf{M}}} \mathrm{Tr}(\mathscr{L}_C^s)
\frac{\left( \frac{1-s}{2e} \right)^{1-s} + \|\mathscr{L}_C\|_{\mathsf{op}
(L^2)}^{1-s}}{1 + (\sum^t_{j=k+1} \eta_j)^{1-s}}, \mbox{ when }0<s<1,
\end{gather*}
and $\Upsilon^K_{2*}\leq \sqrt{c_{\mathsf{M}}} \mathrm{Tr}(\mathscr{L}_C)$ when
$s=1$.

To summarize, for any $t\geq 1$, when $0<s<1$,
\begin{gather}\label{HKBound1000pP0WaG}
\mathbb{E}[\|\hat{\beta}_{t + 1} - \beta^*\|_K^2] \leq C^K
\left[ \left( \sum^t_{k=1} \eta_k \right)^{-2r} + \sum^t_{k = 1}
\frac{\eta_k^2}{1 + (\sum^t_{j=k+1} \eta_j)^{1-s}} \right],
\end{gather}
where
\begin{align*}
C^K=\max&\left\{ \kappa^2 \|g^\dag\|_2^2 \left(\left(\frac{r}{e}\right)^{2r} + \|\mathscr{L}_C
\|_{\mathsf{op} (L^2)}^{2r}\right),\right.\\
&\;\;\left. (2\sqrt{c_{\mathsf{M}}} \|\beta^*\|_2^2 + 2\sigma^2)
\sqrt{c_{\mathsf{M}}} \mathrm{Tr} (\mathscr{L}_C^s)
\left[\left(\frac{1-s}{2e}\right)^{1-s} + \|\mathscr{L}_c\|_{\mathsf{op}
(L^2)}^{1-s}\right] \right\},
\end{align*}
and when $s=1$,
\begin{gather}
\label{HKBound2000z3KpG}
\mathbb{E}[\|\hat{\beta}_{t + 1} - \beta^*\|_K^2] \leq C^K \left[ \left(
\sum^t_{k=1}\eta_k \right)^{-2r}+\frac12 \sum^t_{k=1} \eta_k^2 \right],
\end{gather}
where $C^K = \max\left\{ \kappa^2 \|g^\dag\|_2^2 ((r/e)^{2r} + \|\mathscr{L}_C
\|_{\mathsf{op} (L^2)}^{2r}), 4(\sqrt{c_{\mathsf{M}}} \|\beta^*\|_2^2 +
\sigma^2) \sqrt{c_{\mathsf{M}}} \mathrm{Tr} (\mathscr{L}_C) \right\}$. Bounds
(\ref{HKBound1000pP0WaG}) and (\ref{HKBound2000z3KpG}) are unified by abusing
the notation and denoting $0^0=1$ (so as to make $(\sum^t_{j=k+1}\eta_j)^0=1$
even when the sum is zero). The proof is then completed.
\qed

We are at the position of proving Theorems \ref{thm3HKOnline000EGE8Dv} and
\ref{thm4HKFH0000KvQ6A} as corollaries of Theorem \ref{boundingEstErr0003fOlO}.

\proof[Proof of Theorem \ref{thm3HKOnline000EGE8Dv}]

To apply Theorem \ref{boundingEstErr0003fOlO}, we need only to select a proper bound
$C_3^{\mathsf{S}}$ of step-sizes, to guarantee (\ref{thmHKConvStepSizeCond00056cR0}),
i.e., for $k=1,\ldots, t-1$,
\begin{gather}\label{thmHKConvStepSizeCondCOPY000k8mJYt}
c_{\mathsf{M}} \mathrm{Tr}(\mathscr{L}_K^s)\left[ \left( \frac{2-s}{2e}
\right)^{2-s}+ \kappa^{4-2s} \right] \sum^k_{l=1}\frac{\eta_l^2}{1+\left(
\sum^k_{j=l+1} \eta_j \right)^{2 - s}} \leq \frac12.
\end{gather}
To bound the sum in (\ref{thmHKConvStepSizeCondCOPY000k8mJYt}), we apply Lemma
\ref{lemOnlineStepSize000m7uT2} with $\nu = 2-s>1$ and note that
$0<\theta\leq 1/2$, so $\omega = -\theta<0$. We have
\begin{gather*}
\sum^k_{l=1}\frac{\eta_l^2}{1+\left( \sum^k_{j=l+1} \eta_j \right)^{2 - s}} \leq
C^{\mathsf{OL}} (\nu=2-s, \theta) (k+1)^{-\theta}
\leq \eta_0^s \left( \frac{4^\theta}{\log 2} +9^\theta C_0^{\mathsf{OL}} \right),
\end{gather*}
where the last inequality is just (\ref{tempCOLBound000Viqt}). Therefore, to achieve
(\ref{thmHKConvStepSizeCondCOPY000k8mJYt}) (or equivalently,
\ref{thmHKConvStepSizeCond00056cR0} for Theorem \ref{boundingEstErr0003fOlO}), we
simply need to let
\begin{gather}\label{defC3S000vypUI}
C_3^{\mathsf{S}} = \left\{ 2c_{\mathsf{M}} \mathrm{Tr}(\mathscr{L}^s_K)
\left[ \left( \frac{2-s}{2e} \right)^{2-s} + \kappa^{4-2s} \right]
\left( \frac{4^\theta}{\log 2} + 9^\theta C_0^{\mathsf{OL}} \right)
\right\}^{-1/s}.
\end{gather}

By Theorem \ref{boundingEstErr0003fOlO},
\begin{gather}\label{temp000pg4EN}
\mathbb{E}[\| \hat{\beta}_{t+1} - \beta^* \|_K^2] \leq C^K\left[ \left(
\sum^t_{k=1}\eta_k \right)^{-2r} +\sum^t_{k=1} \frac{\eta_k^2}{1+\left(
\sum^t_{j=k+1} \eta_j \right)^{1-s}} \right].
\end{gather}
We bound the first term in the right-hand side of (\ref{temp000pg4EN})
by Lemma \ref{lemOnlineStepSize000m7uT2} with $\nu=2r$,
\begin{gather}
\label{tempA000E3yNP}
\left(\sum^t_{k=1}\eta_k \right)^{-2r}
\leq \left( \frac{\eta_0 (1-2^{\theta-1})}{1-\theta} \right)^{-2r}
(t+1)^{-2r(1-\theta)}.
\end{gather}
We bound the last sum in (\ref{temp000pg4EN}) by Lemma \ref{lemOnlineStepSize000m7uT2}.
Note that now $\nu=1-s\in(0,1)$ and $0<\theta\leq 1/2$,
so for Lemma \ref{lemOnlineStepSize000m7uT2},
\begin{gather*}
\omega=1-2\theta-\nu+\nu\theta=s(1-\theta)-\theta=\left\{\begin{array}{ll}
\displaystyle -\frac{2r}{2r+s+1},& 2r<1-s,\\
\displaystyle -(1-s)/2,& 2r\geq 1-s.
\end{array}\right.
\end{gather*}
From the definition fo $\theta$, we see that $(\nu,\theta)\in\Omega$ if and only if
$\theta=1/2$, which is equivalent to $2r\geq 1-s$. So,
\begin{gather}
\label{tempB000rbkU3}
\sum^t_{k=1} \frac{\eta_k^2}{1+\left(
\sum^t_{j=k+1} \eta_j \right)^{1-s}}\leq
C^{\mathsf{OL}} \left\{\begin{array}{ll}
(t+1)^{-2r/(2r+s+1)},& 2r<1-s,\\
(t+1)^{-(1-s)/2} \log (t+1), & 2r\geq 1-s.
\end{array}\right.
\end{gather}

We now show that the rates of (\ref{tempA000E3yNP}) is no slower than that of
(\ref{tempB000rbkU3}). In fact, when $2r<1-s$, $-2r(1-\theta)=-\frac{2r}{2r+s+1}$.
When $2r\geq 1-s$, $-2r(1-\theta)=-r\leq -(1-s)/2$. We have proved that
\begin{gather*}
\mathbb{E}\left[ \left\| \hat{\beta}_{t+1} - \beta^* \right\|_K^2 \right]
\leq C_3\left\{\begin{array}{ll}
(t+1)^{-2r/(2r+s+1)},& 2r<1-s,\\
(t+1)^{-(1-s)/2} \log (t+1), & 2r\geq 1-s,
\end{array}\right.
\end{gather*}
where
\begin{gather}\label{C3Specify000JTif}
C_3 = \frac{C^K}{\log 2}\left( \frac{\eta_0 (1-2^{\theta-1})}{1-\theta} \right)^{-2r}
+C^KC^{\mathsf{OL}}.
\end{gather}
\qed

\proof[Proof of Theorem \ref{thm4HKFH0000KvQ6A}]
First, we specify that
\begin{gather}\label{specifyC4S000KtAVF}
C_4^{\mathsf{S}} = \left[ 2c_{\mathsf{M}} \mathrm{Tr} (\mathscr{L}_K^s) \left(
\left( \frac{2-s}{2e} \right)^{2-s} + \kappa^{4-2s} \right) \left(
1+\frac{1-\theta}{e\theta}\right) \right]^{-1}.
\end{gather}
Then, we verify bound (\ref{thmHKConvStepSizeCond00056cR0}) of Theorem
\ref{boundingEstErr0003fOlO}. To this end, we substitute $\eta_t=\eta_0
T^{-\theta}$ with $\theta = (s+2r) / (1+s+2r)$. Note that $2-s\geq 1$,
$\eta_0\leq 1$, and $k\leq T - 1$. We have
\begin{align*}
\sum^k_{l = 1}\frac{\eta_l^2}{1+\left( \sum^k_{j = l+1} \eta_j \right)^{2-s}}
&\leq \sum^k_{l=1} \frac{\eta_0^2 T^{-2\theta}}{1 + (k-l)\eta_0 T^{-\theta}}
\leq \eta_0^2 T^{-2\theta}+\int^{T - 2}_0 \frac{\eta_0^2T^{-2\theta} du}{1 +
\eta_0 T^{-\theta}u}\\
&\leq \eta_0T^{-\theta} + \eta_0 T^{-\theta} \int^{\eta_0T^{-\theta}
(T-2)}_0\frac{du}{1+u}\\
&\leq \eta_0 (1 + T^{-\theta}\log(T^{1-\theta}))\\
&\leq \eta_0\left( 1+\frac{1-\theta}{e\theta} \right),
\end{align*}
where in the last step we used $T^{-\theta}\log T\leq (e\theta)^{-1}$ as
verified in (\ref{polyBoundingLog000zNwX}). So
(\ref{thmHKConvStepSizeCond00056cR0}) is verified. Then by Theorem
\ref{boundingEstErr0003fOlO},
\begin{gather}\label{localBD0009La9igy}
\mathbb{E}[\|\hat{\beta}_{T+1} - \beta^*\|_K^2] \leq C^K \left[
(\eta_0T^{1-\theta})^{-2r} +\sum^T_{k=1} \frac{\eta_0^2 T^{-2\theta}}{1+\left(
\eta_0T^{-\theta} (T-k) \right)^{1-s}} \right].
\end{gather}
The last sum in the above inequality is estimated by
\begin{align}
\sum^T_{k=1} \frac{\eta_0^2 T^{-2\theta}}{1+(\eta_0 T^{-\theta} (T-k))^{1-s}}
&\leq \eta_0^2 T^{-2\theta}+\int^{T -1}_0 \frac{\eta_0^2T^{-2\theta}
du}{1+(\eta_0T^{-\theta}u)^{1-s}} \nonumber \\
&\leq \eta_0 T^{-\theta} + \eta_0T^{-\theta} \int^{T^{1-\theta}}_0
\frac{du}{1+u^{1-s}}. \label{localBD000YpFOr}
\end{align}
Since $0<s\leq 1$, for any $b\geq 1$,
\begin{gather}\label{localBD000ez2rNQ}
\int^b_0\frac{du}{1+u^{1-s}}\leq 1+\int^b_1 u^{s-1} du = 1+\frac{b^s-1}{s}\leq
b^s/s.
\end{gather}
We combine (\ref{localBD0009La9igy}), (\ref{localBD000YpFOr}), and
(\ref{localBD000ez2rNQ}) to have
\begin{gather*}
\mathbb{E}[\|\hat{\beta}_{T+1} - \beta^*\|_K^2] \leq C^k\left( \eta_0^{-2r}
T^{-2r(1-\theta)} + \eta_0 T^{-\theta} +\frac{\eta_0}{s}T^{-\theta+s(1-\theta)}
\right).
\end{gather*}
From $\theta=\frac{s+2r}{1+s+2r}$, we have $-2r(1-\theta)=-\theta+s(1-\theta) =
-2r/(1+s+2r) \geq -\theta$. So the proof is completed with
\begin{gather}\label{C4Specify000FQapF}
C_4 := C^K(\eta_0^{-2r} + \eta_0 + \frac{\eta_0}{s}).
\end{gather} The proof of Theorem \ref{thm4HKFH0000KvQ6A} is complete.
\qed

\section*{Acknowledgments}
Part of the work of Xin Guo is done when he worked at The Hong Kong Polytechnic
University and supported partially by the Research Grants Council of
Hong Kong [Project No.\ PolyU 15305018]. The work of Zheng-Chu Guo is supported by Zhejiang Provincial Natural Science Foundation of China [Project No.\ LR20A010001], National Natural Science Foundation of China [Project No.\ U21A20426, No.\ 12271473], and Fundamental Research Funds for the Central Universities [Project No.\ 2021XZZX001]. The work of Lei Shi is supported partially by Shanghai Science and Technology Program [Project No.\ 21JC1400600 and Project No.\ 20JC1412700] and National Natural Science Foundation of China [Grant No.\ 12171093]. All authors contributed equally to this work and are listed alphabetically. The corresponding author is Lei Shi.

\bibliographystyle{plain}
\bibliography{refBib.bib}

\appendix

\section{Appendix: A Technical Lemma}
In this section of Appendix, we include the following Lemma
\ref{techLemma000avyDpd}, which is commonly used in the literature
\cite{ChenTangFanGuo2022-MR4388513, YingPontil2008-MR2443089, guo2022rates},
with smaller domains of parameters.
Lemma \ref{techLemma000avyDpd} covers the whole domain $(\nu,\theta)\in
(0,\infty)\times(0,1)$, and the proof is not elsewhere available to our
best knowledge. We use Figure \ref{figOmega0006Ne1p} to elucidate
the rates in Lemma \ref{techLemma000avyDpd}.
Figure \ref{figOmega0006Ne1p} is also helpful for understanding
the rates in Lemma \ref{lemOnlineStepSize000m7uT2}, and
Theorems \ref{thm: decreasing step size capacity dependent L2}
and \ref{thm3HKOnline000EGE8Dv}.

\begin{lem}\label{techLemma000avyDpd}
For $b\geq 2$, $0<\theta<1$, and $\nu>0$,
\begin{gather}
\label{intEst10008Atdm}
\int^b_1\frac{u^{-2\theta} du}{1 + (b^{1-\theta} - u^{1-\theta})^\nu}
\leq C_0^{\mathsf{OL}}\left\{
\begin{array}{ll}
b^{\omega}\log b,& \theta=\frac12 \mbox{ and }\nu\leq 1,\mbox{ or, }
\theta\leq \frac12\mbox{ and }\nu=1,\\
b^\omega,& \mbox{otherwise},
\end{array}
\right.
\end{gather}
where $C_0^{\mathsf{OL}}$ is independent of $b$, and
\begin{gather}\label{omegaDef000VeUi}
\omega=\omega(\nu,\theta)=\left\{\begin{array}{ll}
1-2\theta-\nu+\nu\theta,&0<\nu\leq 1\mbox{ and }0<\theta\leq 1/2,\\
-\theta,& \nu\geq 1\mbox{ and }0<\theta\leq \nu/(\nu+1),\\
-\nu(1-\theta),& 1/2\leq\theta<1 \mbox{ and }\theta \geq \nu/(\nu+1).
\end{array}\right.
\end{gather}
In particular, when $\nu\geq 1$, $\omega = -\min\{ \theta, \nu(1-\theta) \}$.
For different combinations of parameters $\nu$ and $\theta$,
the constant $C_0^{\mathsf{OL}}=C_0^{\mathsf{OL}}(\nu,\theta)$
will be specified below in (\ref{COLestEq1000hAx6}), (\ref{COLestEq2000Pe0ow}),
(\ref{COLestEq3000Pe0ow}), (\ref{COLestEq4000QR2aT}), and (\ref{COLestEq5000LW6Z}),
respectively.
\end{lem}
We purposely allow the domains in (\ref{omegaDef000VeUi}) to overlap,
to facilitate the applications.
In spite of the piecewise definition, $\omega(\nu,\theta)$ is a continuous
function on $(0,\infty)\times (0,1)$. We demonstrate the structure of $\omega$
in Figure \ref{figOmega0006Ne1p}.
The estimate in Lemma \ref{techLemma000avyDpd} is tight,
that is, we can reverse the order of the
inequality (\ref{intEst10008Atdm}) by replacing $C_0^{\mathsf{OL}}$ with a smaller
positive constant independent of $b$. We skip the discussion of tightness.

\begin{figure}[h]
\centering
\includegraphics[width=0.6\textwidth]{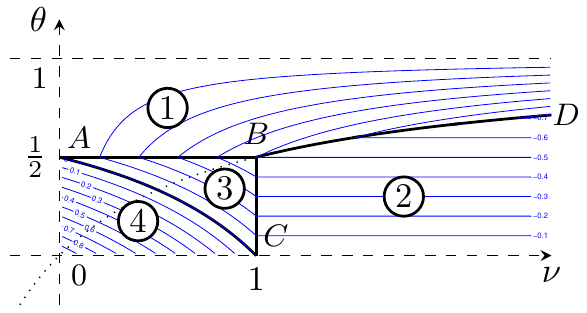}
\caption{Summary of the convergence rates $\omega(\nu,\theta)$.
The domain $(0,\infty)\times (0,1)$ is divided into four regimes
by the solid black lines. Contours of $\omega$ are given in blue lines.
In Regime 1, $\omega = -\nu(1-\theta)$.
In Regime 2, $\omega=-\theta$.
In Regimes 3 and 4, $\omega = 1-2\theta-\nu+\nu\theta=-\nu(1-\theta) +
(1-2\theta)= -\theta+(1-\theta)(1-\nu)$.
The values of $\omega$ continuously extend to the boundaries between regimes.
Arc BD is from the hyperbola $\theta=\nu/(\nu+1)$ which
is extended by the dotted line. $\omega$ approaches its supremum along the
ridge $B\to D$. Logarithm factor in
(\ref{intEst10008Atdm}) only appears on the line
segments AB and BC including point B. In Regime 4 including Arc AC
(which is $\theta=(1-\nu)/(2-\nu)$), $\omega\geq 0$ and the integral in
(\ref{intEst10008Atdm}) does not converge to zero as $b\to\infty$.}
\label{figOmega0006Ne1p}
\end{figure}

\proof[Proof of Lemma \ref{techLemma000avyDpd}]
To verify the estimate (\ref{intEst10008Atdm}), we divide the integral interval
into $[1,b/2]$ and $[b/2, b]$, and denote $\Upsilon_1^{\mathsf{OL}}$
and $\Upsilon_2^{\mathsf{OL}}$ the associated parts of the integral in
(\ref{intEst10008Atdm}), respectively. First,
\begin{align}
\Upsilon^{\mathsf{OL}}_1\leq & \frac{1}{1+(b^{1-\theta}-(b/2)^{1-\theta})^\nu}
\int^{b/2}_1 u^{-2\theta} du \nonumber\\
\leq & \frac{b^{-\nu(1-\theta)}}{(1-2^{\theta-1})^\nu}\times \left\{
\begin{array}{ll}
\frac{(b/2)^{1-2\theta}}{1-2\theta},&\mbox{when }0<\theta<1/2,\\
\log\frac{b}{2},& \mbox{when }\theta=1/2,\\
\frac{1}{2\theta-1},&\mbox{when }1/2<\theta<1,
\end{array}
\right. \nonumber\\
\leq & \frac{1}{(1-2^{\theta-1})^\nu }\times\left\{\begin{array}{ll}
\displaystyle \frac{2^{2\theta-1}}{1-2\theta}
b^{1-2\theta-\nu+\nu\theta},& 0<\theta<1/2,\\
\displaystyle b^{-\nu/2} \log b,& \theta=1/2,\\
\displaystyle \frac{1}{2\theta-1} b^{-\nu (1-\theta)},
& 1/2<\theta<1.
\end{array}\right.\label{Upsilon1Est000eRoc}
\end{align}
Second, to estimate $\Upsilon_2^{\mathsf{OL}}$, we change the variable as
$\xi = b^{1-\theta} - u^{1-\theta}$ to give $d\xi = -(1-\theta) u^{-\theta} du$.
Therefore,
\begin{align*}
\Upsilon_2^{\mathsf{OL}}&=\int^b_{b/2}
\frac{u^{-2\theta} du}{1+(b^{1-\theta} - u^{1-\theta})^\nu}
=\int^{b^{1-\theta} - (b/2)^{1-\theta}}_0
\frac{u^{-\theta} d\xi}{(1+\xi^\nu)(1 - \theta)}\\
&\leq \frac{(b/2)^{-\theta}}{1-\theta}
\int^{b^{1-\theta}}_0
\frac{d\xi}{1+\xi^\nu}.
\end{align*}
Recall that for any $\nu>0$ and $\tau\geq1$,
\begin{gather}
\label{intRational000pZZb}
\int^\tau_0 \frac{d\xi}{1+\xi^\nu}\leq 1+\int^\tau_1\xi^{-\nu} d\xi\leq
\left\{\begin{array}{ll}
1+\frac{\tau^{1-\nu} - 1}{1-\nu}\leq \frac{1}{1-\nu}\tau^{1-\nu},& 0<\nu<1,\\
1+\log\tau,&\nu = 1,\\
1+\frac{1-\tau^{1-\nu}}{\nu-1}\leq\frac{\nu}{\nu-1},&\nu>1.\\
\end{array}\right.
\end{gather}
Therefore,
\begin{gather}\label{Upsilon2Est000lDjk}
\Upsilon_2^{\mathsf{OL}}\leq \left\{\begin{array}{ll}
\displaystyle \frac{2^\theta}{(1-\theta)(1-\nu)} b^{1-2\theta-\nu+\nu\theta},&
0<\nu<1,\\
\frac{2^\theta}{1-\theta} \left( \frac1{\log 2} +1-\theta \right)
b^{-\theta}\log b,& \nu=1,\\
\frac{2^\theta \nu}{(1-\theta)(\nu-1)} b^{-\theta},& \nu>1.
\displaystyle
\displaystyle
\end{array}\right.
\end{gather}
Now we merge (\ref{Upsilon1Est000eRoc}) and (\ref{Upsilon2Est000lDjk})
to derive (\ref{intEst10008Atdm}). Note that the bounds of
$\Upsilon_1^{\mathsf{OL}}$ are divided according to $\theta$, while the bounds
of $\Upsilon_2^{\mathsf{OL}}$ are divided according to $\nu$, so the merging
appears complicated. Figure \ref{figOmega0006Ne1p} provides a clear picture.

\begin{itemize}
\item \textit{Case 1:} $\nu/(\nu + 1)\leq \theta<1$ and $\theta>1/2$. This corresponds
to Regime 1 in Figure \ref{figOmega0006Ne1p}, including the boundary BD but excluding
line segment AB and point B. Below we show that $\omega=-\nu(1-\theta)$.
In fact, now $\Upsilon_1^{\mathsf{OL}} \leq (1-2^{\theta-1})^{-\nu}
(2\theta-1)^{-1} b^{-\nu(1-\theta)}$. When $0<\nu<1$, $\theta>1/2$ implies
$1-2\theta-\nu+\nu\theta < -\nu(1-\theta)$, so
$\Upsilon_2^{\mathsf{OL}}\leq 2^\theta (1-\theta)^{-1}(1-\nu)^{-1} b^{-\nu(1-\theta)}$.
When $\nu=1$, recall that
\begin{gather}\label{polyBoundingLog000zNwX}
\max_{1\leq u<\infty} u^{-a}\log u=\frac{1}{ea},\quad\mbox{for any }a>0,
\end{gather}
where the maximum is achieved at $u = e^{1/a}$. Since $\theta>1/2$,
$b^{-\theta+(1-\theta)}\log b\leq \frac{1}{e(2\theta - 1)}$, and we have
\begin{gather*}
\Upsilon_2^{\mathsf{OL}}\leq \frac{2^\theta}{1-\theta}
\left( \frac{1}{\log 2} + 1-\theta \right) \frac{b^{-(1-\theta)}}{e(2\theta-1)}.
\end{gather*}
When $\nu>1$, the condition $\nu/(\nu+1)\leq \theta$ implies
$-\theta\leq -\nu(1-\theta)$, so $\Upsilon_2^{\mathsf{OL}}
\leq 2^\theta\nu(1-\theta)^{-1}(\nu-1)^{-1} b^{-\nu(1-\theta)}$. We have proved that
\begin{gather*}
\int^b_1\frac{u^{-2\theta}du}{1+(b^{1-\theta} - u^{1-\theta})^\nu}
=\Upsilon_1^{\mathsf{OL}}+ \Upsilon_2^{\mathsf{OL}}
\leq C^{\mathsf{OL}}_0 b^{-\nu(1-\theta)},
\end{gather*}
with
\begin{gather}
\label{COLestEq1000hAx6}
C_0^{\mathsf{OL}}=\frac{(1-2^{\theta-1})^{-\nu}}{(2\theta-1)}
+\left\{\begin{array}{ll}
\displaystyle \frac{2^\theta(\nu+1)}{(1-\theta)|1-\nu|},& \nu>0\mbox{ and }\nu\neq1,\\
\displaystyle \frac{2^\theta}{(1-\theta)(e(2\theta-1))}
\left( \frac{1}{\log 2} + 1-\theta \right),& \nu=1.
\displaystyle
\end{array}\right.
\end{gather}

\item\textit{Case 2:} $\nu>1$ and $0<\theta<\nu/(\nu+1)$. This corresponds to Regime 2
in Figure \ref{figOmega0006Ne1p}, excluding the boundaries BC and BD.
Below we show that $\omega=-\theta$. In fact, in this regime
$\Upsilon_2^{\mathsf{OL}}\leq 2^\theta \nu (1-\theta)^{-1}(\nu-1)^{-1} b^{-\theta}$. When
$0<\theta<1/2$, $1-2\theta-\nu+\nu\theta=-\theta+(1-\theta)(1-\nu)<-\theta$, so
$\Upsilon_1^{\mathsf{OL}}\leq \frac{2^{2\theta-1}}{(1-2^{\theta-1})^\nu(1-2\theta)}
b^{-\theta}$. When $\theta=1/2$, we use (\ref{polyBoundingLog000zNwX}) to see
$b^{-\frac{\nu}{2}+\frac12}\log b\leq \frac{2}{e(\nu-1)}$, so
$\Upsilon_1^{\mathsf{OL}}\leq (1-2^{\theta-1})^{-\nu} \frac{2}{e(\nu-1)}
b^{-\theta}$. When $1/2<\theta<1$, $\theta<\nu/(\nu+1)$ implies
$-\nu(1-\theta)<-\theta$, so $\Upsilon_1^{\mathsf{OL}}\leq
(1-2^{\theta-1})^{-\nu}(2\theta-1)^{-1} b^{-\theta}$. We have proved that
\begin{gather*}
\int^b_1\frac{u^{-2\theta}du}{1+(b^{1-\theta} - u^{1-\theta})^\nu}
=\Upsilon_1^{\mathsf{OL}}+ \Upsilon_2^{\mathsf{OL}}
\leq C^{\mathsf{OL}}_0 b^{-\theta},
\end{gather*}
with
\begin{gather}\label{COLestEq2000Pe0ow}
C_0^{\mathsf{OL}}=\frac{2^\theta\nu}{(1-\theta)(\nu-1)}
+\frac{1}{(1-2^{\theta-1})^\nu}\times\left\{\begin{array}{ll}
2^{2\theta-1}/(1-2\theta),& 0<\theta<1/2,\\
\frac{2}{e(\nu-1)},& \theta=1/2,\\
1/(2\theta-1),& 1/2<\theta<1.
\end{array}\right.
\end{gather}

\item\textit{Case 3:} $0<\theta<1/2$ and $0<\nu<1$. This corresponds to Regimes 3 and 4
in Figure \ref{figOmega0006Ne1p}, including Arc AC but excluding boundaries AB, BC,
and point B. Now it is obvious that $\Upsilon_1^{\mathsf{OL}}+
\Upsilon_2^{\mathsf{OL}}\leq C_0^{\mathsf{OL}} b^{1-2\theta-\nu+\nu\theta}$, with
\begin{gather}\label{COLestEq3000Pe0ow}
C_0^{\mathsf{OL}}=\frac{2^{2\theta-1}}{(1-2^{\theta-1})^\nu(1-2\theta)}
+\frac{2^\theta}{(1-\theta)(1-\nu)}.
\end{gather}

\item\textit{Case 4:} $\theta=1/2$ and $0<\nu\leq 1$. This corresponds to line
segment AB, including point B. Now $-\nu(1-\theta)=-\nu/2$,
$b^{-\nu/2}\log b\geq b^{-\theta} \log b\geq b^{-\theta}\log 2$, and
$1-2\theta-\nu+\nu\theta=-\nu(1-\theta)$. So
\begin{gather*}
\Upsilon_2^{\mathsf{OL}}\leq b^{-\nu(1-\theta) }(\log b)\times
\left\{\begin{array}{ll}
\displaystyle \frac{2^\theta}{1-\theta}\left( \frac1{\log 2} + 1-\theta \right),&
\nu = 1,\\
\displaystyle\frac{2^\theta\nu}{(1-\theta)(\nu-1)\log 2},& 0<\nu<1,
\end{array}\right.
\end{gather*}
and $\Upsilon_1^{\mathsf{OL}}\leq (1-2^{\theta-1})^{-\nu} b^{-\nu/2} \log b$.
So, $\Upsilon_1^{\mathsf{OL}} + \Upsilon_2^{\mathsf{OL}}\leq C_0^{\mathsf{OL}}
b^{-\nu(1-\theta)}\log b$ with
\begin{gather}
\label{COLestEq4000QR2aT}
C_0^{\mathsf{OL}}\leq \frac{1}{(1-2^{\theta-1})^\nu} +\left\{\begin{array}{ll}
\displaystyle \frac{2^\theta}{1-\theta}\left( \frac1{\log 2} + 1-\theta \right),&
\nu = 1,\\
\displaystyle\frac{2^\theta\nu}{(1-\theta)(\nu-1)\log 2},& 0<\nu<1.
\end{array}\right.
\end{gather}

\item\textit{Case 5:} $\nu=1$ and $0<\theta<1/2$. This corresponds to
line segment BC, excluding point B. In this case, $1-2\theta-\nu+\nu\theta=-\theta$. So
$\Upsilon_1^{\mathsf{OL}} + \Upsilon_2^{\mathsf{OL}}\leq C_0^{\mathsf{OL}}
b^{-\theta}\log b$ with
\begin{gather}\label{COLestEq5000LW6Z}
C_0^{\mathsf{OL}}=\frac{2^\theta}{1-\theta}\left( \frac{1}{\log 2} + 1-\theta \right)
+\frac{2^{2\theta-1}}{(1-2^{\theta-1})^\nu(1-2\theta)\log 2}.
\end{gather}
\end{itemize}
\qed

\end{document}